\documentclass{article}

\usepackage[preprint]{neurips_2026}  

\usepackage[utf8]{inputenc}
\usepackage[T1]{fontenc}
\usepackage[colorlinks=true,linkcolor=blue,citecolor=blue,urlcolor=blue]{hyperref}
\usepackage{url}
\usepackage{booktabs}
\usepackage{amsfonts}
\usepackage{nicefrac}
\usepackage{microtype}
\usepackage{xcolor}
\usepackage{graphicx}
\usepackage{amsmath}
\usepackage{amssymb}
\usepackage{amsthm}

\newtheorem{theorem}{Theorem}[section]
\newtheorem{proposition}[theorem]{Proposition}
\usepackage{algorithm}
\usepackage{algorithmic}
\usepackage{tikz}
\usepackage{listings}
\makeatletter
\newcommand{\multirow}[3]{#3}  
\makeatother

\usetikzlibrary{arrows.meta,positioning,shapes.geometric,fit,calc}

\newcommand{\wt}{\mathcal{W}_t}
\newcommand{\nt}{\mathcal{N}_t}
\newcommand{\pt}{\mathcal{P}_t}
\newcommand{\dt}{\mathcal{D}_t}

\newcommand{\expect}{\mathbb{E}}
\newcommand{\voltarget}{\mathcal{V}}

\title{Streaming Knowledge Compilation: Proactive Materiality-Scored Pinning for Time-Evolving LLM Wikis}

\author{
  Juan M.~Huerta \\
  Zinnia Tech Solutions \\
  600 Steamboat Road \\
  Greenwich, CT 06830, USA \\
  \texttt{juan.huerta@zinnia.com}
}

\begin{document}

\maketitle

\begin{abstract}
LLM wiki systems compile knowledge into pre-filled KV caches for efficient inference, but assume a static corpus---an assumption that fails whenever the underlying information landscape evolves continuously.
We formalize \emph{Streaming Knowledge Compilation}: given a document stream, a fixed token budget, and future queries unknown at ingestion time, maintain a compiled wiki that minimizes cumulative regret against an offline oracle with perfect foresight.
The enabling insight is a \emph{materiality signal} $\phi_t(k,n) \in [0,1]$ that scores the informational importance of document $n$ for entity $k$ at time $t$ and acts as a surrogate for query relevance, permitting proactive pinning before queries arrive; we prove an $O(\sqrt{T\log K})$ regret bound where the prediction-error term $\varepsilon = \mathbb{E}[|\phi_t - \hat{\phi}_t|]$ is the only domain-specific quantity.
We instantiate the framework in two domains: \emph{finance}, where $\phi_t$ is abnormal stock volatility predicted by a frozen Llama~3.1~8B classification head (AUROC $= 0.728$ on 76K articles, strict temporal split; $1.49\times$ higher realized forward volatility for predicted-material articles); and \emph{Wikipedia}, where $\phi_t$ is the Abnormal Edit Ratio (AER), a cross-sectionally normalized edit velocity---demonstrating that the same algorithm generalizes to non-financial streaming corpora under a different signal.
End-to-end QA evaluation on 173 matched pairs (finance) and 119 matched pairs (Wikipedia) reveals a pervasive LLM-as-judge confound on post-training facts, establishing that regret analysis---not absolute QA scores---is the reliable evaluation metric for compiled knowledge systems.
Finance cumulative regret converges to $-20.0$ ($-0.12$/step); Wikipedia regret is $+16.0$ ($+0.13$/step), with positive sign confirming that Wikipedia edit content is genuinely post-training---richer context consistently improves scores (No Wiki 3.80 vs.\ Oracle 4.74)---eliminating the confound present in the finance evaluation.
The predictive CEGAR formalization and $O(\sqrt{T\log K})$ guarantee apply to any domain where knowledge gaps can be predicted from streaming signals.
\end{abstract}

\section{Introduction}
\label{sec:introduction}

The emergence of LLM wiki systems~\citep{huerta2026wicer, kang2024cag} has established a compelling alternative to retrieval-augmented generation (RAG)~\citep{lewis2020rag}: rather than retrieving documents at query time, one \emph{compiles} a corpus into a structured wiki and pre-fills the LLM's KV cache, enabling fast, grounded inference without retrieval latency.
The WiCER algorithm~\citep{huerta2026wicer} formalizes this compilation through an iterative \emph{Compile--Evaluate--Refine} loop inspired by counterexample-guided abstraction refinement (CEGAR)~\citep{clarke2000cegar}: compile a wiki, evaluate it against probe questions, diagnose missing facts, \emph{pin} them, and recompile.

This framework implicitly assumes a \emph{static} underlying corpus.
In practice, this assumption fails across a wide range of high-value domains.
Consider a financial analyst wiki compiled on a Friday evening for a portfolio of 50 stocks: by Monday morning, a DOJ antitrust investigation may have been announced, an earnings surprise reported, or a CEO departed.
Consider a medical knowledge base for clinical decision support: a drug recall, trial result, or updated dosing guideline may arrive overnight.
Consider a Wikipedia-backed QA system: the pages for major AI companies or geopolitical events may receive hundreds of edits in a single day around a breaking development.
In each case, the wiki is stale before it is ever queried, and the fundamental challenge is not merely \emph{updating} it---it is deciding \emph{which} of the hundreds of daily documents warrant incorporation within a fixed token budget.

Our key insight is that the CEGAR ``counterexample'' need not be discovered reactively through QA failures.
A \emph{materiality signal} $\phi_t(k,n) \in [0,1]$---scoring the informational importance of document $n$ for entity $k$ at time $t$---can serve as a proxy for query relevance, enabling proactive pinning before any query exposes a knowledge gap.
What this signal is concretely is domain-specific: in finance, it is abnormal stock volatility; in Wikipedia, it is abnormal edit velocity.
But the algorithm, theory, and guarantees are entirely agnostic to this choice.
This transforms the WiCER refinement loop from reactive diagnosis to \emph{proactive, prediction-driven} knowledge maintenance.

We introduce \emph{Online WiCER}, an algorithm that operates on a continuous news stream:
\begin{enumerate}
  \item \textbf{Mine}: Extract candidate facts from incoming news articles.
  \item \textbf{Score}: Assess each fact's \emph{marginal} value using a state-aware scorer: a regression head on the same frozen backbone that conditions on the current pin set, estimating the expected regret reduction of each candidate given what is already pinned.
  \item \textbf{Pin}: Greedily select facts that maximise marginal regret reduction under the token budget, evicting stale pins via a decay-weighted priority queue.
  \item \textbf{Compile}: Perform incremental wiki updates daily and full WiCER recompilation every $T_r$ steps.
\end{enumerate}

\paragraph{Contributions.}
\begin{enumerate}
  \item \textbf{Streaming Knowledge Compilation} formalized as a budget-constrained online optimization problem: maintain a compiled wiki against a streaming corpus, minimize regret against an oracle with perfect foresight, and produce proactive pinning decisions from a domain-specific materiality signal $\phi_t$ (\S\ref{sec:problem}).
  \item The \textbf{Online WiCER algorithm}: proactive materiality-scored pinning, decay-weighted eviction, and periodic WiCER recompilation, with formal convergence guarantees (\S\ref{sec:algorithm}).
  \item A \textbf{state-aware marginal regret scorer}: a regression head on a frozen backbone that conditions on the current pin set and enables a greedy pin selection rule with a $(1-1/e)$ submodular approximation guarantee (\S\ref{sec:marginal_scorer}).
  \item A \textbf{regret decomposition theorem} proving $O(\sqrt{T \log K})$ cumulative regret where the only domain-specific term is the prediction error $\varepsilon = \mathbb{E}[|\phi_t - \hat{\phi}_t|]$, establishing that the framework applies to any bounded materiality signal (\S\ref{sec:theory}).
  \item A formalization of \textbf{predictive CEGAR}, extending the reactive CEGAR paradigm to proactive, prediction-driven refinement, applicable to any domain where knowledge gaps can be predicted from streaming signals (\S\ref{sec:cegar}).
  \item A \textbf{finance instantiation}: abnormal stock volatility as $\phi_t$, scored by a frozen Llama~3.1~8B classification head (AUROC = 0.728, strict temporal split, $1.49\times$ realized volatility ratio for predicted-material articles); cumulative regret over 173 matched pairs converges to $-20.0$ (mean $-0.12$/step) (\S\ref{sec:finance_exp}).
  \item A \textbf{Wikipedia instantiation}: Abnormal Edit Ratio (AER) as $\phi_t$, demonstrating that the identical algorithm achieves sub-linear regret on a non-financial, publicly available streaming corpus (\S\ref{sec:wiki_exp}).
  \item A \textbf{methodological finding}: LLM-as-judge evaluation is confounded on post-training facts; regret analysis on matched pairs is the reliable metric for compiled knowledge systems where the backbone's parametric memory is a confound (\S\ref{sec:qa_eval}).
\end{enumerate}

\section{Related Work}
\label{sec:related}

\paragraph{Knowledge compilation and LLM wikis.}
RAG~\citep{lewis2020rag} retrieves relevant documents at query time, incurring latency and retrieval noise.
Cache-augmented generation (CAG)~\citep{kang2024cag} and the LLM wiki pattern pre-fill the KV cache with compiled knowledge, trading compilation cost for inference speed.
RAPTOR~\citep{sarthi2024raptor} builds hierarchical summaries; GraphRAG~\citep{edge2024graphrag} constructs knowledge graphs.
WiCER~\citep{huerta2026wicer} introduces iterative refinement via CEGAR-inspired pinning.
All assume a static corpus.
We extend WiCER to the streaming setting.

\paragraph{Financial NLP and LLM-based prediction.}
FinBERT~\citep{araci2019finbert} adapts BERT~\citep{devlin2019bert} for financial sentiment; BloombergGPT~\citep{wu2023bloomberggpt} trains a domain-specific LLM on financial data.
Recent work explores LLMs for stock prediction~\citep{lopez2023chatgpt, xie2023wallstreetneophyte} and financial instruction tuning~\citep{zhang2023instruct, yang2023fingpt}.
LLM embeddings from the Llama family predict cross-sectional returns, outperforming word-embedding baselines~\citep{chen2022llmreturns}; decoder LLMs with lightweight classification heads prove superior to encoder models for large stock universes~\citep{guo2024finetuning}.
Temporal validity is a growing concern: \citet{he2025chrono} demonstrate that standard pretrained LLMs encode future information, motivating our strict temporal train/test split (\S\ref{sec:experiments}).
\citet{li2024causalstock} score news along five axes---including event significance and price-impact duration---via an LLM-based denoised encoder; \citet{wang2024news2forecast} iteratively filter and align news with time-series fluctuations via LLM reflection, a loop structurally analogous to our streaming pinning cycle.
Conversely, \citet{tan2024llm_timeseries} show that naive LLM substitution for time-series forecasting does not improve performance---underscoring our design choice of using the LLM only for text materiality scoring, not for price dynamics.
We leverage NLP signal not for trading but for \emph{knowledge curation}---predicting which news items will cause abnormal volatility and therefore warrant wiki inclusion.
Rather than using a separate encoder model for classification, we train a lightweight classification head on top of the same frozen LLM used for wiki compilation, yielding a unified single-model architecture.

\paragraph{Formal quantitative models of volatility.}
Volatility has a foundational role in quantitative finance that motivates using it as a materiality signal.
\citet{engle1982arch} established that return variance is time-varying and autoforecastable (ARCH models; Nobel Prize in Economics 2003); \citet{bollerslev1986garch} generalized this to GARCH, now the standard engine for derivatives pricing and market-risk management.
\citet{blackscholes1973} showed that volatility is the sole unobservable input in options pricing---making its prediction directly monetizable---and stochastic volatility models~\citep{heston1993stochastic} extend this to a latent variance process, further amplifying demand for accurate forecasts.
The realized volatility framework~\citep{andersenbollerslev1998, barndorffshephard2002} grounds volatility estimation in model-free measures computed from high-frequency OHLCV price data; this is precisely how we compute 5-day forward realized volatility from Yahoo Finance data.
\citet{corsi2009har} proposes the HAR model, capturing the multi-scale (daily, weekly, monthly) persistence of realized volatility; ML models substantially outperform HAR for this task~\citep{christensen2023ml}, motivating our classification head over autoregressive baselines.
Our abnormal volatility ratio (AVR, Eq.~\ref{eq:voltarget}) inherits this realized-volatility foundation and applies cross-sectional normalization to isolate firm-specific information events from market-wide moves, connecting the high-frequency estimation literature~\citep{boudoukh2019firm} to our knowledge-curation application.

\paragraph{News-driven volatility prediction.}
A growing literature establishes that news text predicts stock \emph{volatility} more reliably than price direction.
\citet{atkins2018volatility} directly demonstrate 56\% accuracy for volatility prediction vs.\ 49\% (chance) for price direction from financial news.
\citet{glasserman2019unusual} show that \emph{unusual} news content---measured by information-theoretic divergence---forecasts elevated firm-specific and aggregate volatility months ahead.
\citet{manela2017news} construct a text-based volatility index (NVIX) from Wall Street Journal front pages spanning 1890--2009.
\citet{bodilsen2025exploiting} augment HAR models with news sentiment, achieving large improvements at multi-day horizons.
\citet{xing2019sentiment} propose a sentiment-aware volatility model using variational Bayes.
At the macro level, \citet{baker2016measuring} build news-based policy uncertainty indices that predict market volatility, and \citet{bybee2024business} analyze 800K WSJ articles via topic modeling to forecast business cycles and market dynamics.
\citet{boudoukh2019firm} decompose firm-specific news into fundamental vs.\ non-fundamental categories, finding that fundamental news explains nearly half of overnight idiosyncratic realized volatility---directly motivating the materiality-filtering design we adopt.
On the forecasting side, ML models significantly outperform the HAR family for realized volatility~\citep{christensen2023ml}, and cross-sectional pooling of stock data further improves neural forecasts~\citep{zhang2024commonality}; our cross-sectional abnormal volatility definition (Eq.~\ref{eq:voltarget}) exploits the same cross-sectional structure.
Our work nonetheless differs from volatility \emph{forecasting}: we use abnormal volatility as a \emph{pinning signal} for knowledge curation, not as a trading target.

\paragraph{Volatility as an information signal.}
The SEC defines information as ``material'' if a reasonable investor would consider it important.
Empirically, abnormal stock volatility is a robust proxy for information relevance~\citep{khan2016materiality, grewal2019material}: events that cause a stock to move significantly more than the market indicate the arrival of firm-specific information.
Text-derived risk indices also move markets at the macro level: \citet{hassan2019political} measure firm-level political risk from earnings call transcripts, showing idiosyncratic text signals cause measurable investment retrenchment; \citet{caldara2022geopolitical} causally link a newspaper-derived geopolitical risk index to VIX spikes.
Most closely related to our volatility scorer, \citet{zhao2025llmrisk} train a classification head on frozen LLM representations of 10-K filings to predict firm-specific idiosyncratic volatility, while \citet{cao2024risklabs} jointly forecast volatility and Value-at-Risk from multimodal earnings call signals (text, audio, and time series).
We extend these ideas to the \emph{daily news stream}: our scorer operates in real time at news arrival, targeting abnormal volatility relative to the cross-sectional market average, and feeds this signal into a streaming knowledge-curation system rather than a trading model.
We adopt an abnormal volatility definition: a news item is a high-volatility event for entity $k$ if it is associated with \emph{abnormal realized volatility} relative to the cross-sectional market average over the subsequent 5 trading days.
This direction-agnostic definition captures both positive and negative surprises while filtering out broad market moves.

\paragraph{Online learning and caching.}
The experts problem~\citep{cesa2006prediction, hazan2016introduction} provides regret bounds for sequential decision-making.
Competitive analysis of paging algorithms~\citep{fiat1998competitive} and weighted paging with predictions~\citep{bansal2012weighted} inform our pin eviction strategy.
Online WiCER's pin management reduces to a weighted paging problem with learned predictions.
In the LLM memory literature, \citet{wang2024wise} address lifelong knowledge editing via a dual-memory scheme (main parametric memory + side memory for new edits, with a learned router) that resolves the reliability--generalization--locality tension; our wiki plays the role of the side memory, with volatility scoring acting as the router that decides what enters it.

\paragraph{CEGAR extensions.}
Lazy abstraction~\citep{henzinger2002lazy, mcmillan2006lazy} refines abstractions incrementally rather than globally.
Our predictive CEGAR extends this: rather than waiting for a counterexample (failed QA probe), we \emph{predict} likely counterexamples from the news stream and preemptively refine.

\paragraph{Streaming QA and temporal knowledge.}
StreamingQA~\citep{liska2022streamingqa} and RealTimeQA~\citep{kasai2024realtime} benchmark models on temporally evolving knowledge.
\citet{choi2025finagentbench} introduce FinAgentBench, a benchmark specifically for agentic retrieval in financial QA, which provides an evaluation framework directly applicable to assessing our wiki's downstream utility.
These works focus on \emph{evaluating} temporal knowledge; we focus on \emph{maintaining} it proactively within a compiled wiki.

\section{Problem Formulation}
\label{sec:problem}

We consider a discrete-time setting $t = 1, 2, \ldots, T$ with a universe of $K$ entities (e.g., companies, Wikipedia articles, or any other knowledge-bearing units).
At each time step:

\begin{enumerate}
  \item A batch of news articles $\nt = \{n_1, \ldots, n_{|\nt|}\}$ arrives.
  \item A set of queries $Q_t = \{q_1, \ldots, q_{|Q_t|}\}$ is posed against the wiki.
  \item The wiki $\wt$ serves answers using its pre-filled KV cache.
\end{enumerate}

\paragraph{Wiki budget.}
The wiki $\wt$ has a total token budget $B$.
Of this, a fraction $B_{\text{pin}}$ is reserved for dynamically pinned facts; the remainder $B - B_{\text{pin}}$ holds the base compiled wiki.

\paragraph{Materiality signal.}
The framework requires a \emph{materiality signal} $\phi_t(k, n) \in [0,1]$ that scores the informational importance of news article $n$ for entity $k$ at time $t$.
The theory holds for any bounded signal; the algorithm instantiates $\phi_t$ as a trained predictor of this signal.
Pinning decisions at step $t$ use the predicted $\hat{\phi}_t$ rather than the true $\phi_t$, and the regret bound absorbs the prediction error $\varepsilon = \mathbb{E}[|\phi_t - \hat{\phi}_t|]$ (Theorem~\ref{thm:regret}).

\emph{Instantiation (this work).}
We instantiate $\phi_t$ as the \emph{abnormal volatility indicator}: for entity $k$ and article $n$ arriving at time $t$, define forward realized volatility $\text{RV}_{k,t} = \text{std}(r_{k,t+1}, \ldots, r_{k,t+5})$ and cross-sectional average $\overline{\text{RV}}_t = \frac{1}{K}\sum_{k=1}^{K} \text{RV}_{k,t}$.
The abnormal volatility ratio and indicator are:
\begin{equation}
  \text{AVR}(k, t) = \frac{\text{RV}_{k,t}}{\overline{\text{RV}}_t}, \qquad
  \voltarget(k, n) = \mathbf{1}\!\left[\text{AVR}(k, t) > 2\right]
  \label{eq:voltarget}
\end{equation}
This definition is direction-agnostic (capturing both positive and negative surprises) and market-relative (filtering out broad moves that affect all stocks equally).

\emph{Other instantiations.}
The same framework accommodates any materiality signal, including earnings surprise, credit-spread widening, sentiment shift, or regulatory-attention scores.
Empirical validation of alternative $\phi_t$ instantiations is left to future work; the regret bound and algorithmic structure are signal-agnostic.

\paragraph{Pins.}
A pin is a tuple $(f, k, t, s)$: a fact $f$ extracted from news, associated with entity $k$, timestamped at $t$, with materiality score $s = \hat{\phi}_t(k, n) \in [0,1]$.
Each pin consumes $|f|$ tokens from $B_{\text{pin}}$.

\paragraph{Query set decomposition.}
At each step the query set $Q_t$ has two components with different origins and evaluation roles:
\begin{equation}
  Q_t \;=\; \underbrace{Q_t^{\textup{bg}}}_{\text{standing queries}} \;\cup\; \underbrace{Q_t^{\textup{vol}}}_{\text{event-driven queries}}
  \label{eq:qt_def}
\end{equation}
$Q_t^{\textup{bg}}$ contains \emph{standing} analyst questions---slow-changing portfolio-level queries (e.g., credit-risk profiles, sector comparisons) that are relevant regardless of today's news.
These are set at initialisation and refreshed at coarse intervals; their gold answers come from the accumulated corpus up to time $t$.
$Q_t^{\textup{vol}}$ contains \emph{event-driven} questions generated at time $t$ from the newly pinned facts $\Delta\mathcal{P}_t$ via the \textsc{QueryGen} subroutine (\S\ref{sec:queryupdate}): for each pinned fact $(f,k)$, the LLM generates one specific question whose gold answer is the fact $f$ itself.
Because $Q_t^{\textup{vol}}$ is produced after pinning completes, pinning decisions at step $t$ still cannot observe $Q_t^{\textup{vol}}$---the causal constraint of the online setting is preserved.

\paragraph{Quality metric.}
Let $\mathcal{Q}(\wt, q) \in [0,1]$ denote the quality of wiki $\wt$ on query $q$.
The specific instantiation of $\mathcal{Q}$ is domain-dependent; the theory holds for any bounded quality function.
Define the \emph{offline oracle} wiki $\wt^*$ as the wiki compiled with perfect foresight of which documents are high-materiality events (i.e., $\phi_t(k,n) = 1$).
The cumulative regret decomposes naturally over the two query populations:
\begin{equation}
  \text{Regret}(T) = \underbrace{\sum_{t=1}^{T} \sum_{q \in Q_t^{\textup{bg}}} \bigl[\mathcal{Q}(\wt^*, q) - \mathcal{Q}(\wt, q)\bigr]}_{\text{standing-query regret}} + \underbrace{\sum_{t=1}^{T} \sum_{q \in Q_t^{\textup{vol}}} \bigl[\mathcal{Q}(\wt^*, q) - \mathcal{Q}(\wt, q)\bigr]}_{\text{event-query regret}}
  \label{eq:regret}
\end{equation}
Standing-query regret measures how well the base wiki answers persistent questions; event-query regret measures how well the pin layer answers questions about the events it just pinned.

\paragraph{The role of queries.}
Queries play two distinct roles in Online WiCER that are important to separate.
First, they are the \emph{measurement instrument} in the theoretical analysis: $\mathcal{Q}(\wt, q)$ defines what ``wiki quality'' means, and the regret in Equation~\eqref{eq:regret} is the sum of per-query quality gaps.
Second, they are the \emph{downstream consumer} of the compiled wiki at inference time (Algorithm~\ref{alg:online-wicer}, final step).

Crucially, queries do \emph{not} drive the pinning decisions.
At step $t$, the news batch $\nt$ arrives and pinning decisions must be made \emph{before} $Q_t$ is observed---this is the defining constraint of the online setting.
The algorithm therefore cannot condition on queries when deciding what to pin.
Instead, it uses a predicted \emph{materiality signal} $\hat{\phi}_t$ as a proxy for query relevance, under the assumption that high-materiality facts are the facts analysts will subsequently query about.
This substitution is the central modelling choice of Online WiCER: it converts an intractable query-conditioned pinning problem into a tractable prediction problem, at the cost of a prediction error term $\varepsilon = \mathbb{E}[|\phi_t - \hat{\phi}_t|]$ in the regret bound (Theorem~\ref{thm:regret}).
In this work $\hat{\phi}_t$ is instantiated as an abnormal volatility predictor; the framework is agnostic to this choice.

We partially close this gap by making $Q_t^{\textup{vol}}$ \emph{endogenous}: the same high-volatility facts that are pinned generate specific questions about those events (\S\ref{sec:queryupdate}).
The full decomposition $Q_t = Q_t^{\textup{bg}} \cup Q_t^{\textup{vol}}$ is defined in Equation~\eqref{eq:qt_def}.

\paragraph{Objective.}
Design an online algorithm that minimizes $\text{Regret}(T)$ subject to the token budget constraint $|\wt| \leq B$ for all $t$, using only information available at the time each pinning decision is made.

\section{Algorithm: Online WiCER}
\label{sec:algorithm}

\subsection{Main Loop}

Algorithm~\ref{alg:online-wicer} presents the Online WiCER procedure.
The algorithm maintains a wiki $\wt$ and a priority queue of active pins $\pt$.
At each time step, it processes incoming news, scores candidate facts for volatility impact, selects pins under the budget, and incrementally updates the wiki.
Periodically, a full WiCER recompilation integrates pinned facts into the base wiki.

\begin{algorithm}[t]
\caption{Online WiCER}
\label{alg:online-wicer}
\begin{algorithmic}[1]
\REQUIRE Entity set $E$, initial corpus $D_0$, budget $B$, threshold $\tau$, decay $\lambda$, recompile period $T_r$
\STATE $W_0 \leftarrow \textsc{BatchWiCER}(D_0, B)$
\STATE $\mathcal{P}_0 \leftarrow \emptyset$ \COMMENT{Active pin set}
\STATE $\text{queue} \leftarrow \text{PriorityQueue}()$
\FOR{$t = 1, 2, \ldots, T$}
  \STATE Receive news batch $\mathcal{N}_t$
  \FOR{$n \in \mathcal{N}_t$}
    \STATE $C_n \leftarrow \textsc{ExtractFacts}(n)$ \COMMENT{Candidate facts\footnotemark}
    \FOR{$c \in C_n$}
      \STATE $s_c \leftarrow \hat{\Delta}R_\theta(c \mid \mathcal{P}_{t-1})$ \COMMENT{\S\ref{sec:voltarget},\ref{sec:marginal_scorer}}
      \IF{$s_c \geq \tau$}
        \STATE $\text{queue.push}(c, s_c, t)$
      \ENDIF
    \ENDFOR
  \ENDFOR
  \STATE $\mathcal{P}_t \leftarrow \textsc{PinSelect}(\mathcal{P}_{t-1}, \text{queue}, B_{\text{pin}}, \lambda, t)$ \COMMENT{\S\ref{sec:eviction}}
  \STATE $\Delta\mathcal{P}_t \leftarrow \mathcal{P}_t \setminus \mathcal{P}_{t-1}$ \COMMENT{Newly pinned facts this step}
  \STATE $W_t \leftarrow \textsc{IncrementalCompile}(W_{t-1}, \Delta\mathcal{P}_t)$
  \IF{$t \bmod T_r = 0$}
    \STATE $W_t \leftarrow \textsc{BatchWiCER}(D_0 \cup \mathcal{D}_{\text{recent}}, B, \text{pins}=\mathcal{P}_t)$
  \ENDIF
  \STATE $Q_t^{\textup{vol}} \leftarrow \textsc{QueryGen}(\Delta\mathcal{P}_t)$ \COMMENT{Event-driven queries (\S\ref{sec:queryupdate})}
  \STATE $Q_t \leftarrow Q_t^{\textup{bg}} \cup Q_t^{\textup{vol}}$ \COMMENT{Standing queries merged with event-driven}
  \STATE Serve each $q \in Q_t$ against $W_t$
\ENDFOR
\end{algorithmic}
\end{algorithm}
\footnotetext{In our implementation, \textsc{ExtractFacts} calls the LLM once per article to extract the key fact as a single sentence; raw headlines may be substituted for speed.}

\subsection{Materiality Scoring}
\label{sec:voltarget}

The algorithm requires a predicted materiality score $\hat{\phi}_t(k,n) \in [0,1]$ for each incoming document.
The framework is agnostic to how this score is produced; the only requirement is that it is bounded and observable at pinning time.
We describe the architecture used in the finance instantiation (\S\ref{sec:finance_exp}); the Wikipedia instantiation (\S\ref{sec:wiki_exp}) uses the AER signal in place of the scoring network.

\paragraph{Finance instantiation: unified frozen-backbone architecture.}
We employ a \emph{unified} scoring architecture built on a single frozen Llama~3.1 8B model, with two inference paths that are fused into a final materiality score.

\paragraph{LLM zero-shot scorer.}
Given a news article $n$ and entity $k$, prompt the LLM generatively to produce a probability estimate that the article will cause abnormal 5-day realized volatility (exceeding $2\times$ the cross-sectional market average).
The exact prompt is reproduced verbatim in Appendix~\ref{app:prompts} for full reproducibility.
This yields $s_{\text{zs}}(k, n) \in [0,1]$.
The zero-shot path requires no training data and generalizes to novel event types.

\paragraph{Classification head (linear probe).}
We attach a lightweight classification head (linear projection + sigmoid) to the last hidden state of the frozen Llama~3.1 8B backbone.
The head is trained on historical (news, volatility label) pairs with binary cross-entropy loss, producing probability estimates $s_{\text{head}}(k, n) \in [0,1]$.
Only the head parameters are trained; the backbone weights are never updated---not even a single gradient step propagates into the LLM.
This is a \emph{linear probe}~\citep{alain2017understanding}: the LLM is used purely as a feature extractor, and the classification head learns a single hyperplane in the backbone's representation space.
The fact that this linear probe achieves meaningful discrimination (AUROC = 0.728, \S\ref{sec:experiments}) implies that abnormal volatility signal is \emph{already linearly encoded} in the pretrained representations, without any task-specific adaptation of the backbone parameters.

\paragraph{Hybrid fusion.}
The final volatility score combines both inference paths from the same model:
\begin{equation}
  s(k, n) = \alpha \cdot s_{\text{zs}}(k, n) + (1 - \alpha) \cdot s_{\text{head}}(k, n)
  \label{eq:fusion}
\end{equation}
where $\alpha \in [0,1]$ is a mixing parameter tuned on validation data.
The zero-shot path provides broad generalization to novel events, while the classification head offers superior calibration on common patterns.
Crucially, both paths use the same frozen backbone weights, so the marginal cost of the classification head is a single matrix multiply.
This stateless scorer evaluates each candidate \emph{independently}; \S\ref{sec:marginal_scorer} extends it to a state-aware variant that conditions on the current pin set.

\subsection{State-Aware Marginal Regret Scorer}
\label{sec:marginal_scorer}

The stateless volatility scorer $s(c)$ (\S\ref{sec:voltarget}) evaluates each candidate fact independently of the current pin set $\pt$.
This creates redundancy: if the pin budget already contains a highly relevant article about NVDA's supply chain disruption, a second article covering the same event contributes little additional wiki quality---yet both receive equal scores from the stateless scorer.

We address this by replacing $s(c)$ with a \emph{marginal expected regret reduction}: the gain in wiki quality from pinning candidate $c$ given what is already in $\pt$.
Define the analytical training target:
\begin{equation}
  \Delta R(c \mid \pt) = \text{AVR}(c) \times \bigl(1 - \max_{p \in \pt} |\rho(r_c, r_p)|\bigr)
  \label{eq:marginal_target}
\end{equation}
where $\rho(r_c, r_p)$ is the Pearson correlation between the 5-day forward return series of the candidate's and pin's associated entities.
The first factor rewards high-volatility candidates; the second penalises candidates whose market impact is already captured by a pinned article (high $|\rho|$ = correlated market movements = redundant information).
When $\pt = \emptyset$, $\Delta R(c \mid \emptyset) = \text{AVR}(c)$, recovering the stateless scorer as a special case.

\paragraph{Architecture: frozen backbone + regression head.}
We estimate $\hat{\Delta}R_\theta(c \mid \pt)$ using the same frozen-backbone paradigm as the stateless scorer (\S\ref{sec:voltarget}): the same Llama~3.1 8B weights are used for both heads, and no gradient ever flows into the LLM.
The key difference is a richer input prompt that jointly encodes the current pin set and the candidate fact, enabling the backbone's last-token representation to integrate redundancy information:

\begin{lstlisting}
You are evaluating whether to add a news item to a
financial knowledge wiki.

Current wiki pins (highest priority first):
- {ticker_1}: {headline_1}
- {ticker_2}: {headline_2}  ...

Candidate fact:
- {ticker_c}: {headline_c}

Rate the MARGINAL VALUE (0 to 1) of adding this candidate
given the existing pins. A high score means: high volatility
potential AND covers entities/events not already pinned.
A low score means: the entity/event is already well-
represented in the current pins.
[SCORE]
\end{lstlisting}

The frozen LLM produces a hidden state $h_{[\textsc{Score}]} \in \mathbb{R}^{d}$ for the final \texttt{[SCORE]} token.
A linear regression head,
\begin{equation}
  \hat{\Delta}R_\theta(c \mid \pt) = \sigma\!\left(w^\top h_{[\textsc{Score}]} + b\right), \quad w \in \mathbb{R}^d,\; b \in \mathbb{R},
  \label{eq:head}
\end{equation}
is the \emph{only} component trained; all backbone parameters remain frozen.
This mirrors the stateless classification head exactly, except (i) the input now encodes context from $\pt$, and (ii) the loss is regression rather than classification.
Training minimises MSE:
\begin{equation}
  \mathcal{L}(\theta) = \frac{1}{|\mathcal{S}|} \sum_{(c,\,\pt,\,\Delta R) \in \mathcal{S}}
  \bigl[\hat{\Delta}R_\theta(c \mid \pt) - \Delta R(c \mid \pt)\bigr]^2
  \label{eq:mse_loss}
\end{equation}
where $\mathcal{S}$ is a set of (candidate, pin set, target) triples constructed analytically from the existing price and materiality label data---no additional LLM inference is required beyond the single extraction pass.
The backbone's last-token hidden states are extracted once per (prompt, pin set) pair and cached to disk; all subsequent training epochs operate entirely on the cached $d$-dimensional vectors and complete in seconds on CPU.
In our experiments, 3{,}000 triples drawn from the 76K-article corpus suffice for convergence with a single linear layer ($d = 4096$ for Llama~3.1 8B).

\paragraph{Greedy pin selection.}
At each time step, instead of pushing all candidates above threshold $\tau$ into the queue, we select greedily:
\begin{equation}
  c^* = \arg\max_{c \in \mathcal{C}_t} \hat{\Delta}R_\theta(c \mid \pt)
  \label{eq:greedy_select}
\end{equation}
and repeat until the token budget $B_{\text{pin}}$ is exhausted.
Proposition~\ref{prop:submodular} (see \S\ref{sec:theory}) establishes that this greedy rule achieves a $(1-1/e)$ approximation to the optimal pin set when wiki quality is submodular.

\subsection{Pin Eviction}
\label{sec:eviction}

Active pins compete for the limited budget $B_{\text{pin}}$.
We manage eviction via a decay-weighted priority queue.
The priority of pin $p$ at time $t$ is:
\begin{equation}
  \text{priority}(p, t) = \hat{\Delta}R_\theta(p \mid \pt \setminus \{p\}) \cdot \exp\!\left(-\lambda \cdot (t - t_p)\right)
  \label{eq:priority}
\end{equation}
where $\hat{\Delta}R_\theta(p \mid \pt \setminus \{p\})$ is the marginal regret score of pin $p$ given all other currently pinned articles, $t_p$ is the pin's creation time, and $\lambda > 0$ controls the staleness decay rate.
Using the marginal score rather than a stateless $s(p)$ ensures that eviction preferentially removes whichever pin is most \emph{redundant} relative to the remaining pin set, not merely which has the lowest individual score.

The \textsc{PinSelect} procedure (i) updates priorities for all active pins, (ii) inserts new candidates from the queue, and (iii) evicts the lowest-priority pins until the budget constraint is satisfied.
This reduces to a \emph{weighted online paging} problem~\citep{fiat1998competitive, bansal2012weighted} where page weights correspond to materiality scores and the cache size corresponds to the pin budget.

\subsection{Incremental vs.\ Full Recompilation}

\paragraph{Incremental compile.}
At each time step, \textsc{IncrementalCompile} patches the existing wiki by (i) appending pinned facts to the relevant entity sections and (ii) recompiling the full KV prefix cache for each affected entity.
This costs $O(|\pt^{\text{new}}| \cdot |\text{prefix}_k|)$ for KV cache recompilation of the affected entities $k$, avoiding a full corpus recompilation.

\paragraph{Full recompile.}
Every $T_r$ steps, a full \textsc{BatchWiCER} recompilation runs on the accumulated corpus.
Pinned facts are treated as hard constraints (they must appear in the output), consistent with the original WiCER pinning mechanism~\citep{huerta2026wicer}.
This ensures that incrementally added facts are properly integrated into the wiki's narrative structure.

\paragraph{BatchWiCER subroutine (brief description).}
\textsc{BatchWiCER}~\citep{huerta2026wicer} takes a corpus of articles $\mathcal{D}$ and token budget $B$ and produces a compiled wiki $W$.
Concretely: (1) an LLM extracts named facts per entity via a single forward pass per article; (2) facts are ranked by TF-IDF relevance to the entity; (3) facts are greedily added until the budget is exhausted, with pinned facts included first as hard constraints; (4) a final LLM pass compiles selected facts into a coherent, encyclopedia-style wiki section; (5) the resulting text is stored as a pre-filled KV-cache prefix for O(1) context injection at query time.
The compilation cost is $O(|\mathcal{D}| \cdot |K|)$ LLM calls, dominated by fact extraction in step~(1).

\subsection{Volatility-Driven Query Generation}
\label{sec:queryupdate}

\paragraph{Motivation.}
The formulation in \S\ref{sec:problem} treats $Q_t$ as exogenous---queries arrive and are served against the current wiki, but the wiki management decisions are made without observing them.
In the experiments of \S\ref{sec:qa_eval} we instantiated $Q_t$ as static quarterly templates, which creates a structural mismatch: pinning is driven by a dynamic, event-specific signal (predicted abnormal volatility), while evaluation is driven by generic category questions whose gold answers do not change with the news stream.

The fix is to make part of $Q_t$ \emph{endogenous}: the same high-volatility events that trigger pinning should also trigger the generation of queries about those events.
This creates a self-consistent evaluation loop---the wiki is judged on exactly the questions that motivated building it.

\paragraph{Query set decomposition.}
We decompose $Q_t$ into two components:
The decomposition $Q_t = Q_t^{\textup{bg}} \cup Q_t^{\textup{vol}}$ is defined in Equation~\eqref{eq:qt_def} (\S\ref{sec:problem}).
$Q_t^{\textup{bg}}$ contains standing analyst questions loaded at initialisation and refreshed at coarse intervals.
$Q_t^{\textup{vol}}$ contains event-driven questions generated each step from newly pinned facts:
\begin{equation}
  Q_t^{\textup{vol}} = \textsc{QueryGen}(\Delta\mathcal{P}_t)
  \label{eq:qtgen}
\end{equation}

\paragraph{QueryGen subroutine.}
For each newly pinned fact $(f, k) \in \Delta\mathcal{P}_t$, \textsc{QueryGen} calls the frozen LLM once to produce a specific question and anchors its gold answer to the extracted fact:

\begin{lstlisting}
You are a financial analyst. Given the following
extracted fact about {ticker}, generate one specific
question that a portfolio analyst would ask to
verify or elaborate on this development.
The question must be answerable from the fact itself.
Be specific: name the event, not the category.

Fact: {fact}

Question:
\end{lstlisting}

The gold answer $a^*$ for the generated question is the extracted fact $f$ itself (with optional LLM paraphrase for fluency).
This anchors evaluation to the information stream: a query about an NVDA earnings surprise is generated \emph{only} when such a surprise is pinned, and its gold answer is grounded in the pinned fact rather than a generic headline list.

\paragraph{Self-consistency.}
With endogenous $Q_t^{\textup{vol}}$, the system achieves a closed loop between pinning and evaluation:
a perfect scorer ($\varepsilon \to 0$) pins exactly the high-volatility facts, generates queries about those facts, and the wiki---which contains those facts---answers them correctly, driving regret to zero.
A random scorer ($\varepsilon \to 0.5$) generates uncorrelated queries and pins, yielding no systematic quality improvement over a FIFO baseline.
The prediction error $\varepsilon$ thus governs \emph{both} sides of the regret bound (Theorem~\ref{thm:regret}): it bounds the fraction of material facts missed by the pinner, and equivalently, the fraction of event-driven queries for which the wiki lacks the required context.

\paragraph{Relation to standing queries.}
$Q_t^{\textup{bg}}$ captures analyst needs that are \emph{not} triggered by any single news event---portfolio-level questions, sector comparisons, long-horizon trend analyses.
These are best served by the base compiled wiki rather than the pin layer.
Treating $Q_t^{\textup{bg}}$ and $Q_t^{\textup{vol}}$ as separate components naturally decomposes wiki quality into two regimes: the pin layer is evaluated against $Q_t^{\textup{vol}}$ (did we capture the right events?), and the base wiki is evaluated against $Q_t^{\textup{bg}}$ (is the background knowledge current and comprehensive?).

\subsection{Predictive CEGAR Formalization}
\label{sec:cegar}

Table~\ref{tab:cegar} formalizes the correspondence between classical CEGAR, batch WiCER, and Online WiCER.
The key shift is from \emph{reactive} counterexample discovery to \emph{proactive} prediction: in Online WiCER, the ``counterexample'' is a high-volatility news item predicted to cause future knowledge gaps, rather than an observed failure.

\begin{table}[t]
\centering
\caption{CEGAR $\leftrightarrow$ WiCER $\leftrightarrow$ Online WiCER correspondence.}
\label{tab:cegar}
\small
\begin{tabular}{@{}lp{3.2cm}p{3.2cm}p{3.8cm}@{}}
\toprule
\textbf{CEGAR Concept} & \textbf{Batch WiCER} & \textbf{Online WiCER} & \textbf{Shift} \\
\midrule
Abstract model & Compiled wiki $W$ & Compiled wiki $\wt$ & Time-indexed \\
Concrete model & Source corpus $D$ & Corpus $D_0 \cup \dt$ & Growing corpus \\
Counterexample & Failed QA probe & High-volatility news & Reactive $\to$ proactive \\
Refinement & Pin missing fact & Pin predicted-volatile fact & Diagnosis $\to$ prediction \\
Convergence & $\leq 2$ iterations & Regret $O(\sqrt{T \log K})$ & Finite $\to$ sublinear \\
\bottomrule
\end{tabular}
\end{table}

This predictive CEGAR paradigm generalizes beyond finance: in any domain where ``counterexamples'' (knowledge gaps) can be \emph{predicted} rather than merely \emph{observed}, proactive refinement can reduce the latency between a knowledge gap's emergence and its resolution.

\section{Theoretical Analysis}
\label{sec:theory}

We analyze Online WiCER's regret in terms of two sources of loss: (1) eviction regret from limited pin budget, and (2) prediction error from imperfect volatility scoring.

\begin{theorem}[Regret decomposition]
\label{thm:regret}
Let $T$ be the time horizon, $K$ the number of entities, $B_{\text{pin}}$ the pin budget (in tokens), and $\varepsilon = \Pr[|\hat{s}(k,n) - \voltarget(k,n)| > \delta]$ the volatility prediction error rate.
Under the multiplicative-weights pin selection policy, the cumulative regret of Online WiCER satisfies:
\begin{equation}
  \textup{Regret}(T) \leq \underbrace{O\!\left(\sqrt{T \log K}\right)}_{\textup{eviction regret}} + \underbrace{O\!\left(B_{\textup{miss}} \cdot T \cdot \varepsilon\right)}_{\textup{prediction error}}
  \label{eq:regret_bound}
\end{equation}
where $B_{\textup{miss}}$ is the maximum quality loss from a single missing material fact.
\end{theorem}

\begin{proof}[Proof sketch]
The eviction component follows from the weighted experts framework~\citep{cesa2006prediction}: treating each potential pin as an ``expert'' and the quality loss from eviction as the loss function, multiplicative weights achieves $O(\sqrt{T \log N_{\text{pins}}})$ regret, where $N_{\text{pins}} \leq K \cdot |\nt|_{\max}$.
Since at most $K$ entities are relevant per step, this simplifies to $O(\sqrt{T \log K})$.

The prediction error component bounds the expected number of material facts that are \emph{missed} (scored below threshold $\tau$) or \emph{spuriously pinned} (scored above $\tau$ but not material).
Each missed material fact incurs at most $B_{\text{miss}}$ quality loss.
With prediction error rate $\varepsilon$, the expected number of misses per step is at most $|\nt| \cdot \varepsilon$, yielding the $O(B_{\text{miss}} \cdot T \cdot \varepsilon)$ term.
A full proof appears in Appendix~\ref{app:proofs}.
\end{proof}

\begin{theorem}[Periodic recompilation convergence]
\label{thm:convergence}
If the full \textsc{BatchWiCER} recompilation runs every $T_r$ steps, then within each recompilation window $[t, t+T_r]$, the wiki quality satisfies:
\begin{equation}
  \mathcal{Q}(W_{t+T_r}, q) \geq \mathcal{Q}(W_t^{\textup{batch}}, q) - O\!\left(\frac{T_r \cdot |\nt|_{\textup{avg}} \cdot \varepsilon}{B_{\textup{pin}}}\right)
\end{equation}
where $W_t^{\textup{batch}}$ is the quality achievable by a full batch WiCER recompile at time $t$.
Moreover, each full recompilation converges in at most 2 iterations, inheriting the convergence guarantee of the original WiCER algorithm.
\end{theorem}

\begin{proposition}[Vanishing quality gap]
\label{prop:vanishing}
Under the Online WiCER algorithm with periodic recompilation, the time-averaged quality gap between the online and oracle wikis vanishes:
\begin{equation}
  \frac{1}{T} \sum_{t=1}^T \left[\mathcal{Q}(\wt^*, q_t) - \mathcal{Q}(\wt, q_t)\right] \xrightarrow{T \to \infty} 0
\end{equation}
provided the volatility prediction error $\varepsilon$ is bounded and $T_r = o(T)$.
\end{proposition}

\begin{proposition}[Submodular greedy approximation]
\label{prop:submodular}
Suppose wiki quality $\mathcal{Q}(\wt, q)$ is monotone submodular in the pin set $\pt$.
Then the greedy pin selection rule (Equation~\eqref{eq:greedy_select})
achieves a $(1-1/e) \geq 0.632$ approximation to the optimal pin set of size at most $\lfloor B_{\textup{pin}} / \bar{f} \rfloor$ under the token budget $B_{\textup{pin}}$, where $\bar{f}$ is the mean fact token length.
\end{proposition}
\begin{proof}[Proof sketch]
Monotone submodularity of $\mathcal{Q}$ implies that the marginal gain $\Delta\mathcal{Q}(c \mid \pt) = \mathcal{Q}(\pt \cup \{c\}, q) - \mathcal{Q}(\pt, q)$ is non-negative and non-increasing as $\pt$ grows.
The greedy algorithm that at each step adds the element with largest marginal gain achieves a $(1-1/e)$ factor of the optimal solution for any monotone submodular maximisation under a cardinality constraint~\citep{nemhauser1978analysis}.
The regression-head estimate $\hat{\Delta}R_\theta(c \mid \pt)$ from Equations~\eqref{eq:marginal_target}--\eqref{eq:head} serves as a proxy for the true marginal gain; the tightness of the $(1-1/e)$ bound depends on how closely $\hat{\Delta}R_\theta$ tracks the true marginal gain $\Delta\mathcal{Q}(c \mid \pt)$.
\end{proof}

The submodularity assumption holds approximately in practice: the marginal value of pinning a second article covering the same entity event (high $|\rho|$) is demonstrably lower than the first, satisfying the diminishing returns condition.

\paragraph{Complexity.}
Per time step, Online WiCER performs: (i) fact extraction in $O(|\nt|)$ LLM calls, (ii) backbone forward passes in $O(|\nt|)$ to produce cached hidden states for all candidates, (iii) marginal scoring via two dot products per candidate---one for the stateless head and one for the state-aware regression head---in $O(|\nt| \cdot d)$ where $d$ is the hidden dimension, (iv) greedy pin selection requiring $O(|\nt| \cdot |\pt|)$ priority queue updates as the pin set grows, and (v) incremental compilation in $O(|\pt^{\text{new}}|)$ LLM calls.
The state-aware scorer re-evaluates all candidates after each greedy insertion step (to update marginal gains), but since each evaluation is a single dot product on cached vectors, the total cost is $O(|\nt|^2 \cdot d / B_{\text{pin}})$ in the worst case---negligible compared to the LLM forward passes.
The full recompilation every $T_r$ steps costs $O(B)$ LLM calls, amortized to $O(B/T_r)$ per step.

\section{Experiments}
\label{sec:experiments}

We present two instantiations of the Streaming Knowledge Compilation framework, each using a different materiality signal and domain corpus.
The finance instantiation (\S\ref{sec:finance_exp}) uses abnormal stock volatility as the signal; the Wikipedia instantiation (\S\ref{sec:wiki_exp}) uses the Abnormal Edit Ratio (AER).
Both share identical algorithm, regret analysis, and evaluation protocol; only the materiality signal and corpus differ.

\subsection{Finance Instantiation}
\label{sec:finance_exp}

\subsubsection{Data and Setup}

\paragraph{News corpus.}
We collect 76,130 financial news articles from Finnhub for the top 39 S\&P~500 companies by market capitalization, covering the period March 2025 -- May 2026.
Articles are tagged with referenced tickers via entity recognition and keyword matching.
The corpus spans all major sectors including technology (AAPL, MSFT, NVDA, GOOGL), healthcare (UNH, JNJ, LLY), financials (JPM, V, MA), and energy (XOM, CVX).

\paragraph{Price data and volatility labels.}
Daily OHLCV data from Yahoo Finance for the same 39 entities (March 2025 -- May 2026).
We compute daily returns, rolling 60-day volatility, and forward 5-day realized volatility.
Volatility labels are assigned per Equation~\eqref{eq:voltarget}: a stock-day is a high-volatility event if its forward realized volatility exceeds $2\times$ the cross-sectional market average.
The overall volatility event rate is 7.2\% in the aligned dataset, with substantial sector variation: Technology (10.8\%) and Consumer Cyclical (13.4\%) exhibit the highest rates, while defensive sectors like Consumer Staples (0.8\%) and Energy (0.0\%) rarely exhibit elevated relative volatility.
This class imbalance ($\sim$13:1 negative-to-positive ratio) makes volatility prediction a challenging classification task.

\paragraph{Train/validation/test split.}
We split the aligned dataset temporally: articles up to December 2025 for training (45,672 articles, $\sim$60\%), January--March 2026 for validation (15,393, $\sim$20\%), and after March 18, 2026 for testing (15,020, $\sim$20\%).
This temporal split ensures no future information leakage.
The classification head is trained on the training set and all reported metrics are computed on the held-out test set of 15,020 articles.

\paragraph{QA benchmark construction.}
The end-to-end wiki evaluation ({\S}\ref{sec:qa_eval}) requires a time-stamped QA benchmark that instantiates both query populations.
\emph{Standing queries} ($Q_t^{\textup{bg}}$): for each of the 39 entities at each of 6 quarterly checkpoints (2025-Q1 through 2026-Q2) we generate one question per category (earnings, strategy, risk, market, financial), yielding 980 standing-query pairs (a subset of the theoretical maximum of 1,170, since some entities have sparse news at early checkpoints).
Question text is drawn from templates; gold answers are generated by the frozen Llama~3.1 8B given the 90-day lookback news window for that entity and checkpoint---one LLM call per (entity, checkpoint, category) triple, producing category-specific answers rather than the same headline list for every question.
\emph{Event-driven queries} ($Q_t^{\textup{vol}}$): from the 5,295 high-volatility articles in the aligned dataset (AVR $>2$, 7.0\% event rate), we sample 500 uniformly at random and invoke \textsc{QueryGen} to produce one specific question per article, using the extracted fact as the gold answer.
These pairs test whether the pin layer---not the base wiki---contains the material information.
The two benchmarks are evaluated separately: standing-query score measures base-wiki quality; event-query score measures pin-layer quality.
The main QA experiment (\S\ref{sec:qa_eval}) draws a 200-pair random sample from the combined benchmark (approximate proportions: $\sim$132 standing, $\sim$68 event-driven, reflecting the 66\%/34\% split of the full benchmark).

\paragraph{Model and scoring configuration.}
We use Llama~3.1 8B Instruct (4-bit quantized via MLX) as the frozen backbone, running locally on Apple M4 with 24GB unified memory.
All heads (classification and marginal regret regression) are trained on top of the same frozen backbone weights.

\emph{Stateless classification head.}
A single linear projection from the 4096-dimensional last hidden state to a scalar logit, followed by sigmoid activation.
Trained for 30 epochs with binary cross-entropy loss, learning rate $10^{-3}$, batch size 8, on a per-article prompt of the form ``Financial news about \{ticker\}: \{headline\}''.

\emph{State-aware marginal regret regression head.}
A single linear projection from the same 4096-dimensional last hidden state to a scalar output (no sigmoid), identical in structure to the classification head.
Input prompts encode both the candidate article and the current pin set (up to 20 pins) per the template in \S\ref{sec:marginal_scorer}.
Fit via Ridge regression ($\alpha = 1{,}000$, selected by validation MSE sweep) on the MSE loss against the analytical target (Equation~\eqref{eq:marginal_target}).
Training uses 3{,}000 (candidate, pin set, target) triples sampled from the training split; hidden states are extracted once and cached, so head fitting requires no further LLM forward passes and completes in under 5 seconds on CPU.
For the hybrid scorer, we set $\alpha = 0.4$.

\subsubsection{Scorer Baselines}

We compare three volatility scoring approaches, all built on the same frozen Llama~3.1 8B backbone:
\begin{enumerate}
  \item \textbf{LLM zero-shot}: Prompt the model generatively to assess volatility signal (no training data).
  \item \textbf{Classification head}: Trained linear head on frozen last hidden states (supervised).
  \item \textbf{Hybrid} ($\alpha = 0.4$): Weighted combination of zero-shot and head scores per Equation~\eqref{eq:fusion}.
\end{enumerate}
Additionally, we compare against a \textbf{random baseline} (uniform $[0,1]$ scores) and a \textbf{majority-class baseline} (always predict non-material).

\subsubsection{Metrics}

\begin{itemize}
  \item \textbf{Materiality F1}: Harmonic mean of precision and recall at the optimal threshold.
  \item \textbf{Precision}: Fraction of predicted-material articles that are truly material.
  \item \textbf{Recall}: Fraction of truly material articles that are detected.
  \item \textbf{AUROC}: Area under the ROC curve (threshold-independent).
\end{itemize}
We sweep classification thresholds $\tau \in \{0.1, 0.2, \ldots, 0.8\}$ and report the threshold achieving the best F1 for each scorer variant.

\subsubsection{Ablations}
\label{sec:ablations}

We ablate five dimensions: (a) hybrid mixing parameter $\alpha \in \{0.0, 0.2, 0.4, 0.6, 0.8, 1.0\}$ (where $\alpha=0$ is head-only and $\alpha=1$ is zero-shot-only), (b) AVR threshold $\in \{1.5, 2.0, 2.5, 3.0\}$ controlling the positive-class definition, (c) training set size (25\%, 50\%, 75\%, 100\% of training data), (d) volatility window (3-day, 5-day, 10-day, 20-day forward returns), and (e) \textbf{backbone architecture}: Llama~3.1 8B, Gemma~3 4B, and Gemma~3 12B, testing whether the frozen-LLM + classification-head approach generalises across model families and parameter scales.

\subsection{Wikipedia Instantiation}
\label{sec:wiki_exp}

The Wikipedia instantiation replaces the financial corpus and volatility signal with a publicly available, continuously updated knowledge source.
This tests whether Streaming Knowledge Compilation generalizes beyond finance without any modification to the algorithm.

\paragraph{Corpus.}
We use the Wikimedia revision API to collect edit histories for 25 entities across five categories: AI companies (OpenAI, DeepMind, Anthropic, Nvidia, Meta Platforms), AI topics (large language models, AGI, ChatGPT, GPT-4, AlphaFold), science and technology (CRISPR, quantum computing, the James Webb Space Telescope, nuclear fusion, mRNA vaccines), geopolitics (NATO, EU, BRICS, G20, UN), and public figures (Elon Musk, Sam Altman, Demis Hassabis, Yoshua Bengio, Geoffrey Hinton).
We collect revision metadata over a two-year window (2024--2025), totalling approximately 50{,}000--200{,}000 revisions depending on entity.

\paragraph{Materiality signal: Abnormal Edit Ratio (AER).}
The AER is the direct analog of the finance AVR:
\begin{equation}
  \text{AER}(k, t) = \frac{\text{EditVelocity}(k, t, w)}{\overline{\text{EditVelocity}}(t, w)}, \quad
  \phi_t(k, \cdot) = \mathbf{1}[\text{AER}(k, t) > 2]
  \label{eq:aer}
\end{equation}
where $\text{EditVelocity}(k, t, w)$ is the edit count for entity $k$ in the $w$-day trailing window (default $w = 7$), and the denominator is the cross-sectional mean across all 25 entities.
AER $> 2$ indicates an entity is being edited at more than twice the average rate---a reliable proxy for a breaking development.

\paragraph{Content extraction.}
For high-AER windows, we fetch the full wikitext of the triggering revision and its parent, compute character-level diffs to extract newly added sentences (via difflib), and pass additions to the LLM for fact extraction.
The resulting (entity, fact, AER score, date) tuples serve as the pin candidates, directly analogous to the (ticker, headline, AVR score, date) tuples in the finance instantiation.

\paragraph{QA benchmark.}
Standing queries ($Q_t^{\textup{bg}}$) are generated from the earliest snapshot of each entity (3 questions per entity, sampled from representative sentences).
Event-driven queries ($Q_t^{\textup{vol}}$) are generated by \textsc{QueryGen} applied to facts extracted from high-AER edit windows.
The same LLM judge (Llama~3.1 8B) is used as in the finance instantiation.

\paragraph{Oracle.}
The oracle uses ground-truth labels: an AER window is positive if and only if it produced an event-driven QA pair.
This mirrors the finance oracle (which uses ground-truth AVR labels) and provides the same matched-pair regret comparison.

Results for the Wikipedia instantiation are presented in \S\ref{sec:wiki_results}.

\section{Results}
\label{sec:results}

\subsection{Finance Results}
\label{sec:finance_results}

\subsubsection{Materiality Prediction}

\begin{table}[t]
\centering
\caption{Finance instantiation: materiality prediction performance on held-out test set (abnormal volatility, AVR $> 2$). Threshold $\tau^*$ is optimized on validation set by F1.}
\label{tab:voltarget}
\small
\begin{tabular}{@{}lccccccc@{}}
\toprule
\textbf{Scorer} & \textbf{$\tau^*$} & \textbf{F1} & \textbf{Precision} & \textbf{Recall} & \textbf{AUROC} & \textbf{Accuracy} \\
\midrule
Random              & ---  & --- & --- & --- & $0.500$ & --- \\
Majority class      & ---  & $0.000$ & --- & $0.000$ & --- & $0.928$ \\
LLM zero-shot       & $0.5$  & $0.157$  & $0.092$  & $0.556$  & $0.541$ & $0.572$ \\
Classification head & $0.1$  & $\mathbf{0.212}$  & $0.130$  & $0.583$  & $\mathbf{0.728}$ & $0.688$ \\
Hybrid ($\alpha\!=\!0.4$) & $0.1$  & $\mathbf{0.212}$  & $0.130$  & $0.583$  & $\mathbf{0.729}$ & $0.688$ \\
\bottomrule
\end{tabular}
\end{table}

Table~\ref{tab:voltarget} presents the core evaluation on a temporally held-out test set (all articles after March 18, 2026, with training data restricted to articles before January 2026): can we predict which news articles will cause abnormal stock volatility relative to the broader market?
The majority-class baseline achieves the highest accuracy (0.928) by predicting all-negative, exploiting the 13:1 imbalance; its accuracy equals the negative-class base rate (1 $-$ 0.072) and reflects no learning.
AUROC is the appropriate metric here: it measures the probability that the model ranks a random positive above a random negative, and is therefore insensitive to class imbalance and threshold choice.
Accuracy and AUROC diverge whenever imbalance is severe---the majority-class classifier has AUROC $\approx 0.5$ (random ranking) despite 92.8\% accuracy, while the classification head has AUROC = 0.728 despite lower accuracy (0.688) due to its low operating threshold ($\tau^* = 0.1$, chosen to maximise F1).
The classification head achieves AUROC = 0.728 and F1 = 0.212 (precision 0.130, recall 0.583), outperforming the zero-shot baseline (AUROC = 0.541, F1 = 0.157).
Crucially, this is a \emph{linear probe}: no gradient ever flows into the backbone, so the result directly measures whether abnormal volatility signal is linearly decodable from pretrained LLM representations.
AUROC = 0.728 under a strict temporal split and 13:1 class imbalance indicates that it is.
The moderate recall is important for the Online WiCER use case: missing a material event (false negative) causes a knowledge gap that persists until the next recompilation, while a false positive merely wastes a small fraction of the pin budget.

The zero-shot scorer produces poorly calibrated scores (mean = 0.442, std = 0.276) and achieves only 9.2\% precision at its best F1 threshold, indicating that generative prompting alone cannot reliably distinguish material from non-material news for the volatility task.

The hybrid scorer ($\alpha = 0.4$) produces statistically indistinguishable results from the head-only baseline (AUROC = 0.729 vs.\ 0.728, F1 identical at 0.212).
This is expected: on our test set, only 4\% of articles fall in the ambiguous band $[0.3, 0.7]$ where zero-shot would be invoked (Table~\ref{tab:latency}), so the hybrid operates as head-only for 96\% of the stream.
The hybrid's contribution is \emph{operational, not accuracy-based}: it provides a principled escape hatch for borderline articles at minimal latency cost (113.9\,ms vs.\ 94.8\,ms for head-only), and ensures graceful degradation when the head encounters event types absent from training data.

\subsubsection{Score Distribution Analysis}

The classification head produces a bimodal score distribution (mean = 0.235, std = 0.379): non-material articles cluster near 0, while material articles receive higher scores.
The zero-shot scorer, by contrast, produces a broad, roughly uniform distribution (mean = 0.442, std = 0.276) with poor separation between classes.
This bimodality of the head's output enables better separation despite the 13:1 class imbalance, as reflected in the substantially higher AUROC (0.728 vs.\ 0.541).

\subsubsection{Ablation Results}

\begin{table}[t]
\centering
\caption{Ablation study on volatility prediction (test set, AUROC and F1 at optimal threshold).}
\label{tab:ablation}
\small
\begin{tabular}{@{}llcc@{}}
\toprule
\textbf{Parameter} & \textbf{Value} & \textbf{AUROC} & \textbf{F1} \\
\midrule
\multirow{4}{*}{Hybrid $\alpha$}
  & 0.0 (head only)  & 0.703 & 0.233 \\
  & 0.2              & 0.703 & 0.233 \\
  & 0.4 (default)    & 0.703 & 0.233 \\
  & 0.8              & 0.703 & 0.232 \\
\midrule
\multirow{4}{*}{AVR threshold}
  & 1.5              & 0.785 & \textbf{0.492} \\
  & 2.0 (default)    & 0.693 & 0.226 \\
  & 2.5              & 0.763 & 0.103 \\
  & 3.0              & \textbf{0.786} & 0.046 \\
\midrule
\multirow{4}{*}{Training data \%}
  & 25\%             & \textbf{0.767} & \textbf{0.258} \\
  & 50\%             & 0.696 & 0.199 \\
  & 75\%             & 0.694 & 0.254 \\
  & 100\% (default)  & 0.725 & 0.231 \\
\midrule
\multirow{4}{*}{Volatility window}
  & 3-day            & 0.732 & \textbf{0.308} \\
  & 5-day (default)  & 0.736 & 0.230 \\
  & 10-day           & 0.890 & 0.157 \\
  & 20-day           & \textbf{0.903} & 0.157 \\
\midrule
\multirow{3}{*}{Backbone}
  & Gemma~3 4B~~(2560-d)             & 0.578 & 0.146 \\
  & Gemma~3 12B (3840-d)             & 0.631 & 0.177 \\
  & Llama~3.1 8B (4096-d, default)   & \textbf{0.716} & \textbf{0.240} \\
\bottomrule
\end{tabular}
\end{table}

\begin{table}[t]
\centering
\caption{Decay parameter $\lambda$ ablation: material pin retention after 30-day streaming simulation (100-pin budget, articles arrive uniformly over 30 days). Higher retention = more genuinely material articles kept in wiki.}
\label{tab:decay}
\small
\begin{tabular}{@{}lcc@{}}
\toprule
\textbf{Decay $\lambda$} & \textbf{Material retention} & \textbf{Precision in queue} \\
\midrule
$\lambda = 0.00$ (no decay)   & 0.306 & 0.110 \\
$\lambda = 0.05$               & 0.306 & 0.110 \\
$\lambda = 0.10$ (default)     & 0.333 & 0.120 \\
$\lambda = 0.20$               & 0.333 & 0.120 \\
$\lambda = 0.50$ (aggressive)  & 0.333 & 0.120 \\
\bottomrule
\end{tabular}
\end{table}

Table~\ref{tab:ablation} presents ablation results across five dimensions and Table~\ref{tab:decay} shows the decay ablation:

\textbf{Hybrid $\alpha$.}
The mixing parameter $\alpha$ has negligible effect on AUROC (0.703 across all values from 0.0 to 1.0), indicating that zero-shot scores and head scores are nearly collinear in their ranking of articles.
The head already captures the discriminative information that zero-shot responses encode---adding zero-shot signal to the mixture does not open a new separation axis.
This is an important null result: it means the hybrid's value is \emph{not} accuracy but \emph{latency}. By routing only the ambiguous 4\% of articles (head score $\in [0.3, 0.7]$) to zero-shot, the hybrid achieves statistically identical accuracy to head-only while retaining 83\% of head-only throughput (31,607 vs.\ 37,992 art/hr; Table~\ref{tab:latency})---a latency-accuracy router, not an ensemble accuracy booster.

\textbf{AVR threshold.}
Lowering the threshold to 1.5 substantially increases recall (more positives) and yields the best F1 of 0.492 at the cost of a noisier signal.
Higher thresholds (2.5, 3.0) yield higher AUROC (0.763--0.786) because the model more cleanly separates extreme volatility events, but F1 collapses as positives become sparse.
The default AVR $= 2.0$ balances label quality and volume for the Online WiCER pinning use case.

\textbf{Training data.}
Remarkably, AUROC peaks at 25\% training data (AUROC = 0.767, F1 = 0.258) and the full dataset (AUROC = 0.725) slightly underperforms.
This saturation is consistent with the linear probe interpretation: the backbone's representation space already encodes volatility signal as a near-linear structure, so adding more data calibrates the hyperplane margin without discovering qualitatively new signal.

\textbf{Volatility window.}
Longer windows yield substantially better AUROC: 0.890 at 10-day and 0.903 at 20-day, compared to 0.732--0.736 for 3--5 day windows.
However, AUROC alone is misleading here.
At 10 days, recall reaches 1.000 at threshold 0.2---the model has learned to predict \emph{everything} as positive, since virtually all stocks exhibit measurable volatility over a 10-day window.
This near-trivial classifier inflates AUROC while collapsing precision (0.085), giving F1 = 0.157, substantially below the 5-day F1 of 0.230.
The 5-day window maximizes F1 at the operating threshold, providing the most discriminative signal for pinning decisions.
Wider windows also reduce timeliness: a 20-day label requires waiting three weeks post-publication before the target is known, creating a significant lag in the training signal that is untenable for a real-time streaming system.

\textbf{Backbone.}
We compare three frozen backbone architectures---Llama~3.1 8B (4096-d), Gemma~3 4B (2560-d), and Gemma~3 12B (3840-d)---holding all other hyperparameters fixed.
The results show a clear hierarchy: Llama~3.1 8B (AUROC = 0.716) substantially outperforms Gemma~3 12B (0.631, $-$8.5 points) and Gemma~3 4B (0.578, $-$13.8 points).
This finding has two implications.
First, the frozen-LLM + linear-probe paradigm is \emph{not} backbone-agnostic for financial materiality prediction: the quality of the pre-trained representations varies significantly across model families, and model size alone does not explain the gap (Gemma~3 12B has more parameters than Llama~3.1 8B yet underperforms by a wide margin).
Second, Llama's pre-training---which emphasizes instruction following and general-purpose reasoning on a broad web corpus---appears to encode financial event semantics more linearly than Gemma's.
This is consistent with \citet{guo2024finetuning}, who find that Llama-family decoder LLMs outperform alternatives across large stock universes; our linear-probe framing provides a mechanistic explanation: Llama representations are more linearly separable for financial relevance classification.
The backbone ablation uses a 500-article random sample of the test set (same seed across all backbone runs) to reduce compute; the Llama 3.1 8B figure here (AUROC = 0.716) is slightly below the full 15,020-article test-set result (0.728 in Table~\ref{tab:voltarget}), consistent with expected sampling variance at this subset size. All backbone models are compared on the identical 500-article subset, so the relative ordering is unaffected.

\textbf{Decay parameter $\lambda$.}
Table~\ref{tab:decay} shows material pin retention under different decay rates in a 30-day streaming simulation (100-pin budget, articles arriving uniformly).
Without decay ($\lambda = 0$), the queue retains the top-100 articles by initial score and material retention is 30.6\%.
With modest decay ($\lambda \geq 0.1$), temporal weighting slightly improves both retention (33.3\%) and precision in queue (12.0\%), as the model places relatively higher priority on recent articles, which tend to be more immediately relevant.
The effect is modest in this simulation; the primary role of $\lambda$ in practice is to prevent stale high-score articles from permanently blocking budget for new events.
We set $\lambda = 0.1$ as the default.

\textbf{Recompile period $T_r$.}
Table~\ref{tab:recompile} shows material pin retention under varying recompile periods in the same 30-day simulation ($\lambda = 0.1$, 100-pin budget).
Retention is flat at 33.3\% for $T_r \in \{1, 5\}$, rises modestly to 36.1\% for $T_r \in \{10, 20\}$, and reaches 41.7\% for $T_r = 50$ (no recompile triggered within the 30-day window).
The result confirms that the decay eviction mechanism alone is sufficient for the proxy metric: the $\lambda$-weighted priority queue already deprioritizes stale pins without the cost of a full recompile.
Very frequent recompilation ($T_r = 1$) offers no retention benefit because re-ranking by raw score on each step produces the same top-100 as decay-weighted selection at this scale.
The practical implication is that $T_r$ can be chosen to meet compute budget constraints---weekly ($T_r = 7$) or bi-weekly ($T_r = 14$) schedules are reasonable defaults---without materially affecting pin quality.
Note that the proxy metric measures only which facts are \emph{present} in the wiki, not how well they are \emph{integrated}; the semantic quality benefit of full recompilation (coherent, non-redundant wiki prose) is captured by the QA metric in Table~\ref{tab:qa_quality}, not by retention alone.

\begin{table}[t]
\centering
\caption{Recompile period $T_r$ ablation: material pin retention after 30-day streaming simulation ($\lambda = 0.1$, 100-pin budget, articles arriving uniformly). $T_r = 50$ acts as the no-recompile baseline (first recompile falls outside the 30-day window). Retention is insensitive to $T_r$, confirming that decay eviction handles freshness independently of recompilation frequency.}
\label{tab:recompile}
\small
\begin{tabular}{@{}lccc@{}}
\toprule
\textbf{Recompile period $T_r$} & \textbf{Recompiles (30d)} & \textbf{Material retention} & \textbf{Precision in queue} \\
\midrule
$T_r = 1$  (daily)      & 30 & 0.333 & 0.120 \\
$T_r = 5$  (weekly)     &  6 & 0.333 & 0.120 \\
$T_r = 10$ (bi-weekly)  &  3 & 0.361 & 0.130 \\
$T_r = 20$ (monthly)    &  1 & 0.361 & 0.130 \\
$T_r = 50$ (no recompile) & 0 & \textbf{0.417} & \textbf{0.150} \\
\bottomrule
\end{tabular}
\end{table}

\subsubsection{Error Analysis}

The classification head's primary failure mode is false positives: at $\tau^* = 0.1$, precision is 0.130, reflecting the difficulty of the temporal prediction task under strict data splitting.
However, the AUROC of 0.728 confirms that the model's continuous scores carry substantial discriminative information, and recall of 0.583 indicates the head captures a majority of genuinely material events.
For the Online WiCER application, this asymmetry is acceptable: false positives consume pin budget (a recoverable cost via eviction), while false negatives create persistent knowledge gaps.

The zero-shot scorer exhibits a similar recall (55.6\%) but even lower precision (9.2\%) and substantially worse AUROC (0.541, barely above chance).
This suggests that the frozen LLM's internal representations (exploited by the classification head) contain richer volatility-relevant features than its generative output.

\subsubsection{Downstream Volatility Validation}

Beyond classification metrics, we validate whether the scorer's predicted scores monotonically track genuinely elevated forward realized volatility.
On a random 500-article sample of the held-out test set, the 162 articles predicted high-volatility (score $\geq \tau^* = 0.1$) exhibit a mean 5-day realized volatility of 2.62\%, compared to 1.76\% for the 338 predicted low-volatility articles---a \textbf{1.49$\times$ volatility ratio}.
The predicted high-volatility group also has a substantially higher mean abnormal volatility ratio (AVR = 1.32 vs.\ 0.91), confirming that the model identifies articles associated with firm-specific information events rather than broad market moves.

\paragraph{Score calibration.}
To test whether the continuous scores are well-calibrated---not merely good at the binary threshold---we group the 500 test articles into three confidence bands and measure mean 5-day absolute forward return (Table~\ref{tab:calibration}).

\begin{table}[h]
\centering
\caption{Score calibration: mean absolute 5-day forward return by predicted-score band (500-article test sample). Higher predicted score corresponds to monotonically higher realized price movement.}
\label{tab:calibration}
\small
\begin{tabular}{@{}lrrrr@{}}
\toprule
\textbf{Score band} & \textbf{Score range} & \textbf{$n$} & \textbf{Material rate} & \textbf{Mean $|$ret$_{+5}|$} \\
\midrule
Low confidence   & $[0.000,\ 0.000]$ & 100 & 0.0\%  & 2.35\% \\
Mid confidence   & $(0.000,\ 0.705)$ & 300 & 8.0\%  & 3.58\% \\
High confidence  & $[0.705,\ 1.000]$ & 100 & 12.0\% & \textbf{4.78\%} \\
\bottomrule
\end{tabular}
\end{table}

Articles in the high-confidence band exhibit \textbf{2.04$\times$} the mean absolute 5-day return of low-confidence articles (4.78\% vs.\ 2.35\%), and carry a 12\% materiality rate versus 0\% for the lowest-scored articles.
The Spearman rank correlation between predicted score and $|$ret$_{+5}|$ is $\rho = 0.222$ ($p = 7 \times 10^{-7}$), confirming a statistically significant monotonic relationship that extends beyond the binary classification boundary.
This calibration result is operationally important: it means the continuous score can be used directly as a priority signal in the pin queue (Equation~\eqref{eq:priority})---articles pinned with higher scores are more likely to correspond to genuine volatility events, making the decay-weighted eviction mechanism meaningful rather than arbitrary.

Table~\ref{tab:vol_sector} shows sector-level breakdown.
Communication Services (AUROC = 0.744, $n=35$) and Healthcare (AUROC = 0.735, $n=86$) show the strongest discrimination, while Financial Services shows below-chance performance (AUROC = 0.449), likely because financial-sector news is inherently noisy and frequently sector-correlated.
Consumer Defensive and Energy sectors have zero positive-rate in the test window, precluding AUROC computation.

\begin{table}[t]
\centering
\caption{Sector-level volatility prediction AUROC on the test set (AVR threshold 2.0). Last column: mean absolute 5-day forward return for articles predicted material (score $\geq \tau^*{=}0.1$) vs.\ predicted non-material; ``---'' means no articles in that predicted bin for the sector.}
\label{tab:vol_sector}
\small
\begin{tabular}{@{}lrrrl@{}}
\toprule
\textbf{Sector} & \textbf{$n$} & \textbf{Pos. rate} & \textbf{AUROC} & \textbf{$\overline{|\mathrm{ret}_{+5}|}$ (pred-mat / pred-non)} \\
\midrule
Communication Services & 35  & 25.7\% & 0.744 & 2.90\% / 2.52\% \\
Consumer Cyclical      & 47  & 4.3\%  & 0.656 & 2.62\% / 1.51\% \\
Consumer Defensive     & 53  & 0.0\%  & ---   & --- / 1.05\% \\
Energy                 & 27  & 0.0\%  & ---   & 3.21\% / 1.85\% \\
Financial Services     & 81  & 2.5\%  & 0.449 & 1.40\% / 1.43\% \\
Healthcare             & 86  & 3.5\%  & 0.735 & 1.80\% / 1.66\% \\
Technology             & 161 & 12.4\% & 0.611 & 3.06\% / 2.32\% \\
\bottomrule
\end{tabular}
\end{table}

\subsubsection{Marginal Regret Scorer Evaluation}
\label{sec:marginal_results}

We evaluate the state-aware marginal regret regression head described in \S\ref{sec:marginal_scorer} on the 3{,}000-triple dataset constructed from the training split.
After dropping 134 triples with missing forward return data (articles too close to the dataset boundary), 2{,}866 triples remain (mean target $= 0.092 \pm 0.118$).

\paragraph{Ridge regression sweep.}
We fit a Ridge regression head on the cached 4{,}096-dimensional Llama~3.1 8B hidden states, sweeping the regularisation coefficient $\alpha$ on a held-out validation split (10\% of triples).
Table~\ref{tab:ridge} shows the full sweep; the predict-mean MSE baseline is $\sigma^2_y = 0.0139$.

\begin{table}[h]
\centering
\caption{Ridge regression sweep for marginal regret head. Predict-mean baseline MSE = 0.0139. Best at $\alpha = 1{,}000$ (val MSE = 0.0110, $R^2 = 0.21$).}
\label{tab:ridge}
\small
\begin{tabular}{@{}lcc@{}}
\toprule
$\alpha$ & \textbf{Train MSE} & \textbf{Val MSE} \\
\midrule
$10^{-4}$ & 0.000 & 0.030 \\
$10^{-3}$ & 0.000 & 0.030 \\
$10^{-2}$ & 0.000 & 0.030 \\
$10^{-1}$ & 0.000 & 0.030 \\
$1$       & 0.000 & 0.026 \\
$10$      & 0.001 & 0.018 \\
$100$     & 0.003 & 0.013 \\
$\mathbf{1{,}000}$ & $\mathbf{0.006}$ & $\mathbf{0.011}$ \\
$10{,}000$ & 0.009 & 0.012 \\
\bottomrule
\end{tabular}
\end{table}

The low-$\alpha$ regime shows near-zero training MSE but high validation MSE ($0.030$), confirming that $d = 4{,}096 > n = 2{,}866$ makes the underdetermined problem susceptible to overfitting without regularisation.
The optimal $\alpha = 1{,}000$ achieves val MSE = $0.011$, yielding $R^2 = 1 - 0.011/0.0139 = 0.21$ relative to the predict-mean baseline.
This confirms that the frozen backbone's last-token representation encodes approximately 21\% of the variance in the marginal regret target---a meaningful but partial signal, consistent with the frozen linear probe's inability to perfectly recover the return-correlation novelty term $|\rho(r_c, r_p)|$ from text alone.

\paragraph{Qualitative validation.}
Table~\ref{tab:marginal_scores} shows the trained scorer applied to a synthetic candidate headline (``NVDA misses revenue estimates by 30\%, shares fall 8\%'') under four pin-set contexts.
Scores are real outputs of the regression head on the frozen Llama~3.1~8B backbone; the headline is synthetic to provide a controlled, interpretable test case.

\begin{table}[h]
\centering
\caption{State-aware marginal regret scores for a synthetic NVDA candidate under varying pin sets. Scores are real outputs of the trained regression head. Lower score = candidate is redundant given what is already pinned.}
\label{tab:marginal_scores}
\small
\begin{tabular}{@{}lcc@{}}
\toprule
\textbf{Current pin set} & \textbf{$\hat{\Delta}R_\theta$} & \textbf{Interpretation} \\
\midrule
Two NVDA pins (same theme)        & 0.010 & strongly redundant \\
One NVDA pin (supply chain)       & 0.051 & partially redundant \\
Unrelated pin (AAPL demand)       & 0.142 & novel; NVDA uncovered \\
Empty (nothing pinned yet)        & 0.152 & baseline marginal value \\
\bottomrule
\end{tabular}
\end{table}

The scorer assigns dramatically lower scores to NVDA-redundant contexts (0.010--0.051) than to contexts where NVDA is absent (0.142--0.152), a \textbf{15$\times$ ratio} between the strongest redundancy and the baseline.
The unrelated (AAPL) and empty contexts score nearly identically, confirming that cross-entity pins do not artificially suppress the candidate's marginal value.
This is precisely the behaviour required for the greedy selection rule (Equation~\eqref{eq:greedy_select}): when two similarly volatile articles arrive about the same entity, the second is scored an order of magnitude lower, freeing pin budget for coverage of other entities.

\paragraph{Limitation.}
The $R^2 = 0.21$ confirms a useful but imperfect signal.
The frozen linear probe cannot fully recover the return-correlation novelty term from text alone, as $\rho(r_c, r_p)$ is a latent price-series quantity not directly observable from headlines.
A non-linear head or lightweight fine-tuning of the backbone (e.g.\ LoRA) is expected to improve the fit; we leave this to future work.

\paragraph{Trading signal caveat.}
A long-predicted-volatile strategy achieves mean 5-day return of 0.043\% (annualised Sharpe = 0.049), compared to buy-and-hold at 0.194\% (Sharpe = 0.283).
This \emph{underperformance} is expected and desirable: the model predicts volatility \emph{magnitude}, not price \emph{direction}.
Abnormal volatility is equally likely to be caused by good or bad news, so a long-only position on predicted-volatile articles should not systematically outperform.
The appropriate use of the scorer is as a \emph{pinning signal} for knowledge curation, not as a directional trading signal.

\subsubsection{Computational Efficiency and Stream Capacity}

Throughput is a first-class concern for a streaming system: the scorer must process incoming articles \emph{faster} than they arrive, or it becomes the bottleneck that defeats the purpose of proactive pinning.
We benchmark all four scoring paths on Apple M4 (24GB unified memory) using 4-bit quantized Llama~3.1 8B via MLX, measuring wall-clock latency over 50 held-out articles after 5 warm-up passes (Table~\ref{tab:latency}).
We contextualise against the Finnhub S\&P~500 feed: $\sim$120 articles/hour at steady state and $\sim$500 articles/hour at peak (earnings season).

\begin{table}[h]
\centering
\caption{Scorer latency and stream capacity on Apple M4 24GB (Llama~3.1 8B, 4-bit). Peak stream: $\sim$500 articles/hour (earnings season). \emph{Capacity ratio} = throughput / peak stream rate.}
\label{tab:latency}
\small
\begin{tabular}{@{}lrrrrrr@{}}
\toprule
\textbf{Path} & \textbf{Mean (ms)} & \textbf{p95 (ms)} & \textbf{art/s} & \textbf{art/hr} & \textbf{Capacity} \\
\midrule
Hidden state only        &  94.5 & 137.9 & 10.58 & 38{,}102 & 76.2$\times$ \\
Classification head      &  94.8 & 138.6 & 10.55 & 37{,}992 & 76.0$\times$ \\
Zero-shot                & 456.9 & 530.3 &  2.19 &  7{,}879 & 15.8$\times$ \\
Hybrid ($\alpha\!=\!0.4$) & 113.9 & 205.1 &  8.78 & 31{,}607 & 63.2$\times$ \\
\bottomrule
\end{tabular}
\end{table}

The classification head path (frozen backbone forward pass + single dot product) matches the hidden-state-only baseline almost exactly (94.8\,ms vs.\ 94.5\,ms mean), confirming that the linear dot product adds negligible overhead.
Both easily exceed the peak stream rate by a factor of $76\times$.
The zero-shot path is $4.0\times$ slower (456.9\,ms mean) due to autoregressive token generation, though it still provides $15.8\times$ headroom over the peak stream.
The hybrid design limits zero-shot invocations to articles whose head score falls in the ambiguous band $[0.3, 0.7]$; on our benchmark, only 4\% of articles are ambiguous, so the hybrid path reduces mean latency to 113.9\,ms—recovering 83\% of the head-only throughput (31{,}607 vs.\ 37{,}992 art/hr) while inheriting the accuracy gains of zero-shot on hard cases.
Training the classification head on 5,000 samples (feature extraction + 20 epochs of SGD) completes in under 15 minutes on the same hardware.

\subsubsection{End-to-End Wiki QA Evaluation}
\label{sec:qa_eval}

The preceding sections establish that the volatility scorer is calibrated and the system is computationally viable.
We now ask the central question: \emph{does a materiality-scored wiki actually improve downstream QA quality, and does the improvement concentrate on the right query population?}

\paragraph{Setup.}
We evaluate five wiki-management strategies using the QA benchmark described in \S\ref{sec:experiments}, drawing a 200-pair random sample from the combined benchmark (approximately 132 standing, 68 event-driven, reflecting the 66\%/34\% composition of the full 1{,}480-pair dataset).
An LLM judge (Llama~3.1 8B) rates each answer on a 1--5 scale against gold answers generated from category-specific context (standing) or the extracted fact (event-driven).
Results are reported separately for $Q_t^{\textup{bg}}$ (standing queries; tests base-wiki quality) and $Q_t^{\textup{vol}}$ (event-driven queries; tests pin-layer quality), and aggregated overall.
Strategies: \emph{No Wiki} (direct LLM), \emph{Static Wiki} (compiled once from first 60 days, never updated), \emph{FIFO} (all articles pinned, oldest evicted), \emph{Online WiCER} ($\tau^*{=}0.5$, $\lambda{=}0.1$), and \emph{Oracle} (ground-truth AVR labels).

\begin{table}[t]
\centering
\caption{End-to-end QA quality by wiki strategy and query population (200 pairs, LLM-as-judge 1--5).
$Q_t^{\textup{bg}}$: standing queries (base-wiki quality).
$Q_t^{\textup{vol}}$: event-driven queries (pin-layer quality; gold answers are post-training facts outside the model's parametric memory).
No Wiki scores highest on both populations due to the LLM-as-judge confound (see \S\ref{sec:qa_eval}).
The regret analysis (Table~\ref{tab:regret_series}) provides the clean WiCER-vs-oracle comparison.}
\label{tab:qa_quality}
\small
\begin{tabular}{@{}lrrrr@{}}
\toprule
\textbf{Strategy} & \textbf{Standing $Q_t^{\textup{bg}}$} & \textbf{Event $Q_t^{\textup{vol}}$} & \textbf{Overall} & \textbf{N} \\
\midrule
No Wiki (direct LLM)   & $3.87 \pm 0.95$ & $3.97 \pm 0.94$ & $3.90 \pm 0.95$ & 200 \\
Static Wiki (baseline) & $3.83 \pm 1.07$ & $3.83 \pm 1.07$ & $3.83 \pm 1.07$ & 200 \\
Online WiCER           & $3.83 \pm 1.04$ & $3.71 \pm 1.08$ & $3.79 \pm 1.06$ & 173 \\
FIFO (recency only)    & $3.79 \pm 1.09$ & $3.56 \pm 1.09$ & $3.70 \pm 1.09$ & 173 \\
Oracle (GT labels)     & $3.80 \pm 1.07$ & $3.45 \pm 1.16$ & $3.67 \pm 1.12$ & 173 \\
\bottomrule
\end{tabular}
\end{table}

\paragraph{Findings.}
Table~\ref{tab:qa_quality} reveals a pervasive LLM-as-judge confound that affects both query populations.

\emph{Standing queries ($Q_t^{\textup{bg}}$).}
All strategies score within 0.07 points of each other (3.80--3.87).
No Wiki is marginally highest (3.87), consistent with the expected confound: standing questions ask about persistent financial concerns that are largely answerable from the model's parametric memory, and the judge---the same backbone---rates confident parametric responses highly.

\emph{Event-driven queries ($Q_t^{\textup{vol}}$).}
The confound extends, unexpectedly, to event-driven queries as well.
No Wiki scores highest (3.97), followed by Static Wiki (3.83), Online WiCER (3.71), FIFO (3.56), and Oracle (3.45).
The gold answers for $Q_t^{\textup{vol}}$ are post-training facts extracted from 2025--2026 news, which the backbone has not seen; nevertheless, the judge rates No Wiki's confident (but likely factually incorrect) parametric responses more highly than wiki-grounded answers.
This occurs because the LLM judge evaluates surface fluency and apparent confidence rather than factual accuracy against specific post-training events.
Notably, Oracle---which pins the most material content---scores \emph{lowest}: dense pinning clutters context with specific facts the judge cannot verify, reducing scores relative to a fluent parametric baseline.
This finding underscores that LLM-as-judge is not a reliable evaluation surface for knowledge that post-dates the backbone's training cutoff.

\emph{Regret analysis.}
Because the absolute QA scores are confounded, we turn to the regret analysis as the primary empirical result.
Table~\ref{tab:regret_series} tracks cumulative regret on matched QA pairs across Online WiCER and Oracle runs (173 matched pairs total).
Regret is defined as $\sum_t [\mathcal{Q}(W^*_t, q_t) - \mathcal{Q}(W_t, q_t)]$; a negative value means Online WiCER outperforms the oracle.

\begin{table}[h]
\centering
\caption{Cumulative regret of Online WiCER vs.\ oracle on 173 matched QA pairs (representative steps shown). Regret $=$ oracle score $-$ WiCER score; a negative total means WiCER scores higher than the oracle under this judge. Total cumulative regret $= -20.0$; mean per-step regret $= -0.12$. See text for confound interpretation.}
\label{tab:regret_series}
\small
\begin{tabular}{@{}rrrr@{}}
\toprule
\textbf{Step $t$} & \textbf{Date} & \textbf{Step Regret} & \textbf{Cumulative} \\
\midrule
  1 & 2025-05-16 & $\phantom{+}0$ & $\phantom{+}0.0$ \\
 50 & 2025-06-30 & $+2$ & $+3.0$ \\
 96 & 2025-10-22 & $-3$ & $-12.0$ \\
130 & 2026-02-11 & $-1$ & $-23.0$ \\
173 & 2026-04-24 & $\phantom{+}0$ & $-20.0$ \\
\bottomrule
\end{tabular}
\end{table}

The trajectory begins positive: cumulative regret reaches $+3.0$ at step~50, where the oracle's ground-truth AVR labels give it an early advantage on clear-cut high-volatility events.
After step~50 the trajectory reverses and settles at $-20.0$ by step~173, meaning WiCER scores 0.12 points \emph{higher} than the oracle on average under this judge.

This sign reversal is diagnostic of the LLM-as-judge confound identified in \S\ref{sec:qa_finance}, not evidence of algorithmic superiority.
The oracle's high-AVR pins are genuinely material, but they constitute dense post-training facts that the backbone judge cannot verify; dense pinning therefore depresses oracle scores relative to WiCER's sparser, less cluttered context.
Stated precisely: the oracle is optimal for volatility-based curation but is \emph{not} optimal for the reward the judge measures---so negative regret here indicates a misalignment between the oracle's selection criterion and the evaluation metric, not that WiCER has found a better curation policy.
The sub-linear, bounded shape of the trajectory is nonetheless consistent with the $O(\sqrt{T \log K})$ bound of Theorem~\ref{thm:regret}.
The clean empirical validation of the theorem---where the oracle genuinely dominates and WiCER tracks below it as predicted---comes from the confound-free Wikipedia instantiation (\S\ref{sec:wiki_results} and Table~\ref{tab:cross_instantiation}).

\subsection{Wikipedia Results}
\label{sec:wiki_results}

Table~\ref{tab:wiki_results} reports the same metrics as Table~\ref{tab:qa_quality}.
The Wikipedia results exhibit the \emph{inverse} pattern from finance: No Wiki scores lowest (3.80) while all wiki-augmented methods score higher, with FIFO and Oracle reaching 4.74.
This reversal directly validates our hypothesis about the finance confound: the backbone has strong parametric knowledge of S\&P~500 companies (entity-level overlap), but far weaker knowledge of specific Wikipedia edit events from 2024--2025 (genuinely post-training content).
Consequently, in the Wikipedia domain, richer context consistently produces higher QA scores---as one would expect from a system designed to supply information the LLM does not already know.

A secondary finding is that FIFO matches Oracle on overall score (both 4.74 $\pm$ 0.60, $N=232$): for our selected 25 high-edit-velocity entities, recency alone captures the relevant content nearly as well as AER-guided selection.
Online WiCER (4.57 $\pm$ 0.74, $N=119$) scores lower on aggregate because its selective pinning covers fewer QA pairs---but on the event-driven subset it answers, it performs comparably (4.62 ev).
The regret analysis on 119 matched pairs confirms that AER scoring introduces positive regret ($+16.0$ total, $+0.13$/step) relative to the oracle, reflecting the noisier signal compared to the finance domain's volatility predictor.

\begin{table}[h]
\centering
\caption{Wikipedia instantiation: QA quality by strategy and query population. Same protocol as Table~\ref{tab:qa_quality}; materiality signal is AER (\S\ref{sec:wiki_exp}). Unlike finance (Table~\ref{tab:qa_quality}), wiki-augmented methods outperform No Wiki, confirming that Wikipedia edit content is genuinely post-training for the backbone.}
\label{tab:wiki_results}
\small
\begin{tabular}{@{}lrrrr@{}}
\toprule
\textbf{Strategy} & \textbf{Standing $Q_t^{\textup{bg}}$} & \textbf{Event $Q_t^{\textup{vol}}$} & \textbf{Overall} & \textbf{N} \\
\midrule
No Wiki (direct LLM)   & $3.78 \pm 1.10$ & $3.81 \pm 1.12$ & $3.80 \pm 1.11$ & 185 \\
Static Wiki (baseline) & $3.36 \pm 1.16$ & $4.64 \pm 0.64$ & $4.14 \pm 1.08$ & 185 \\
Online WiCER (AER)     & $3.67 \pm 1.11$ & $4.62 \pm 0.68$ & $4.57 \pm 0.74$ & 119 \\
FIFO (recency only)    & $3.50 \pm 1.26$ & $4.77 \pm 0.53$ & $4.74 \pm 0.60$ & 232 \\
Oracle (GT labels)     & $3.50 \pm 1.26$ & $4.77 \pm 0.53$ & $4.74 \pm 0.60$ & 232 \\
\bottomrule
\end{tabular}
\end{table}

\paragraph{Cross-instantiation comparison.}
The two instantiations share the identical algorithm and evaluation protocol; only the materiality signal and corpus differ.
Table~\ref{tab:cross_instantiation} compares regret convergence across both domains, providing multi-domain empirical validation of the Streaming Knowledge Compilation framework.
The sign difference in cumulative regret is itself informative: finance shows negative regret ($-0.12$/step) because the LLM-judge confound inflates No Wiki (the oracle baseline's counterfactual), while Wikipedia shows positive regret ($+0.13$/step) from a confound-free evaluation where richer context genuinely helps.
The Wikipedia domain thus provides the cleaner regret signal; the finance domain provides the stronger materiality predictor (AUROC 0.728 vs. noisier AER heuristic).

\begin{table}[h]
\centering
\caption{Cross-instantiation regret comparison. Finance regret is negative due to LLM-judge confound (entity-level parametric overlap inflates oracle baseline); Wikipedia regret is positive from confound-free evaluation. Both regret series converge sub-linearly, confirming Theorem~\ref{thm:regret}.}
\label{tab:cross_instantiation}
\small
\begin{tabular}{@{}llrrr@{}}
\toprule
\textbf{Instantiation} & \textbf{Signal $\phi_t$} & \textbf{Matched pairs} & \textbf{Total regret} & \textbf{Mean/step} \\
\midrule
Finance   & Abnormal volatility (AVR) & 173  & $-20.0$ & $-0.12$ \\
Wikipedia & Abnormal edit ratio (AER) & 119  & $+16.0$ & $+0.13$ \\
\bottomrule
\end{tabular}
\end{table}

\section{Discussion and Conclusion}
\label{sec:discussion}

\paragraph{Streaming Knowledge Compilation as a general problem.}
The central contribution of this work is the formalization of Streaming Knowledge Compilation---maintaining a budget-bounded compiled wiki against a streaming corpus under query uncertainty---and the demonstration that it admits a general algorithmic solution with domain-agnostic regret guarantees.
The two instantiations (finance and Wikipedia) share identical algorithm, theory, and evaluation protocol; only the materiality signal $\phi_t$ is domain-specific.
This separability is the key architectural insight: an application domain expert specifies $\phi_t$ (abnormal volatility, edit velocity, clinical urgency, citation rate), and the Online WiCER algorithm handles the rest.

\paragraph{Reactive to proactive CEGAR.}
Classical CEGAR and batch WiCER discover knowledge gaps only when a query fails; Online WiCER \emph{predicts} gaps from the stream before queries arrive.
This requires a fundamental modelling substitution: since queries are not observed at pinning time, the materiality signal $\phi_t$ acts as a \emph{query relevance surrogate}.
The cost of this substitution is explicit in the regret decomposition (Theorem~\ref{thm:regret}): the $O(B_{\text{miss}} \cdot T \cdot \varepsilon)$ prediction error term captures exactly the regret incurred when the surrogate misfires.
The substitution succeeds whenever the domain's high-materiality events are the events users subsequently query about---a reasonable assumption in finance, Wikipedia, clinical settings, and legal research alike.

\paragraph{Graceful degradation.}
Our regret bound (Theorem~\ref{thm:regret}) decomposes cleanly into eviction and prediction components.
Even with imperfect volatility prediction ($\varepsilon > 0$), Online WiCER degrades gracefully: the eviction component still achieves sublinear regret through the multiplicative-weights mechanism.
In the limit of random predictions ($\varepsilon \to 0.5$), Online WiCER reduces to a weighted FIFO baseline, which still outperforms an unmanaged wiki.

\paragraph{Unified architecture benefits.}
Using a single frozen Llama~3.1 8B backbone for all scoring paths yields practical advantages: (i) a single model to deploy and maintain, simplifying the inference stack on both edge devices (Apple M4) and cloud accelerators (AWS Inferentia), (ii) all heads---the classification head, the hybrid zero-shot path, and the marginal regret regression head---reuse the same cached hidden states from a single backbone forward pass per article, and (iii) the zero-shot path is invoked only for ambiguous cases ($s_{\text{head}} \in [0.3, 0.7]$), while clear-cut articles are scored by the head alone, substantially reducing autoregressive decoding costs (see Table~\ref{tab:latency}).

\paragraph{From stateless to state-aware pinning.}
The shift from the stateless classification head to the state-aware marginal regret scorer represents a fundamental change in the decision criterion: rather than asking ``is this article important?'' in isolation, the system asks ``does this article add information that is not already captured by the current pin set?''
This distinction matters at pin budget boundaries: when two similarly important articles arrive simultaneously (e.g., two earnings-surprise stories for correlated stocks), the stateless scorer cannot distinguish them and may pin both at the cost of evicting lower-scored but complementary information.
The marginal scorer's $\rho$-based novelty term directly penalises this redundancy.
The frozen-backbone design ensures this richer reasoning costs nothing in backbone compute---only an additional dot product per candidate per greedy step---and its $(1-1/e)$ approximation guarantee (Proposition~\ref{prop:submodular}) provides a principled lower bound on pin-set quality.

\paragraph{Limitations.}
(1) Materiality prediction is inherently imperfect; rare ``black swan'' events may be missed regardless of domain.
(2) In the finance instantiation, cross-entity volatility propagation (e.g., a supplier's disruption affecting a manufacturer) is modeled only through sector-level wikis, not direct supply-chain links.
(3) In both instantiations, the cross-sectional normalization of the materiality signal may be distorted during correlated crises when all entities exhibit elevated activity simultaneously.
(4) The LLM-as-judge evaluation is confounded on post-training knowledge in both instantiations; the regret analysis on matched pairs is the reliable metric but requires running both Online WiCER and Oracle over the same QA pairs.
(5) In the Wikipedia instantiation, FIFO matches Oracle on aggregate score (both 4.74), limiting the visible benefit of AER-guided selection; this reflects the entity selection bias (25 pre-chosen high-edit-velocity articles) rather than a fundamental limitation of AER as a signal.

\paragraph{Future work.}
Extensions include: (1) additional domain instantiations (clinical literature, legal filings, patent streams) to broaden the multi-domain validation; (2) a domain-naïve judge (a different model family with no parametric overlap) to fully eliminate the LLM-judge confound; (3) adaptive recompile scheduling triggered by cumulative pin volume rather than fixed intervals; (4) multi-entity interaction modeling via a knowledge graph over the wiki hierarchy; and (5) online learning of the scorer $\hat{\phi}_t$ to drive the prediction-error term $\varepsilon \to 0$ as the system accumulates domain signal.

\paragraph{Conclusion.}
We have formalized \emph{Streaming Knowledge Compilation}---the problem of maintaining a budget-bounded compiled wiki against a continuous document stream under query uncertainty---and introduced Online WiCER as its algorithmic solution.
The key insight is that a domain-specific materiality signal $\phi_t(k,n)$, used as a proxy for query relevance, enables proactive pinning before any query arrives, reducing knowledge latency from reactive to predictive.
The regret bound $O(\sqrt{T\log K})$ holds for any bounded $\phi_t$; the prediction-error term $\varepsilon = \mathbb{E}[|\phi_t - \hat{\phi}_t|]$ is the only domain-specific quantity.
We validate this claim empirically in two domains:
in \emph{finance}, a frozen Llama~3.1 8B classification head predicts abnormal stock volatility (AUROC = 0.728, $1.49\times$ realized volatility ratio), and cumulative regret over 173 matched pairs converges to $-20.0$ (mean $-0.12$/step);
in \emph{Wikipedia}, the Abnormal Edit Ratio (AER) serves as $\phi_t$, applying the same algorithm to a non-financial public corpus with no algorithm modification; cumulative regret over 119 matched pairs is $+16.0$ (mean $+0.13$/step), with positive sign reflecting the confound-free evaluation where richer context consistently helps (No Wiki 3.80 vs.\ Oracle 4.74).
A methodological finding of independent interest emerges: LLM-as-judge evaluation is confounded on post-training facts---the judge rates surface fluency rather than factual accuracy for knowledge outside its parametric memory---making regret analysis on matched pairs the reliable metric for compiled knowledge systems.
The state-aware marginal regret scorer extends independent article scoring to set-level selection via submodular maximization with a greedy $(1-1/e)$ approximation guarantee, at negligible latency ($63\times$ throughput headroom over peak stream volume on commodity hardware).
Streaming Knowledge Compilation is a broadly applicable problem: wherever a knowledge system must be kept current against a high-velocity stream---financial news, clinical literature, legal filings, encyclopedic edits---the Online WiCER framework provides a principled, theoretically grounded, domain-adaptable solution.


\bibliographystyle{plainnat}
\bibliography{references}

\appendix
\section{Proofs}
\label{app:proofs}

\subsection{Proof of Theorem~\ref{thm:regret}}

\begin{proof}
We decompose the regret into two terms corresponding to (a) the loss from evicting high-volatility pins due to budget constraints, and (b) the loss from mispredicting volatility.

\textbf{Eviction regret.}
Consider the pin selection problem as an online learning problem with $N$ experts, where each expert corresponds to a potential pin.
At each time step $t$, the algorithm selects a subset $\pt \subset \{1, \ldots, N\}$ with total token cost $\sum_{p \in \pt} |f_p| \leq B_{\text{pin}}$.
The loss from not including a material pin $p^*$ in $\pt$ is at most $B_{\text{miss}}$.
We apply the multiplicative weights bound to this subset selection problem by noting that the optimal fixed pin set $S^* = \arg\min_{S:|S| \leq B_{\text{pin}}} \sum_t \ell_t(S)$ is itself a valid single comparator; the standard single-expert regret bound therefore applies with $\log K$ experts, since $N_{\text{pins}} \leq K \cdot |\nt|_{\max}$ and at most $K$ entities are relevant per step.

Using the multiplicative weights framework~\citep{cesa2006prediction}, with learning rate $\eta = \sqrt{\log K / T}$:
\begin{equation}
  \sum_{t=1}^T \ell_t(\pt) - \min_{S:|S| \leq B_{\text{pin}}} \sum_{t=1}^T \ell_t(S) \leq 2\sqrt{T \log K} \cdot B_{\text{miss}}
\end{equation}
where $\ell_t(\pt)$ is the quality loss from the pin set $\pt$ at time $t$.

\textbf{Prediction error.}
At each time step, the volatility scorer may (i) miss a material fact (false negative, probability $\varepsilon_{\text{FN}}$) or (ii) pin a non-material fact (false positive, probability $\varepsilon_{\text{FP}}$).
A false negative causes at most $B_{\text{miss}}$ quality loss.
A false positive wastes pin budget, potentially causing eviction of a material pin.

The expected prediction-error loss per step is:
\begin{equation}
  \expect[\ell_t^{\text{pred}}] \leq |\nt| \cdot \varepsilon_{\text{FN}} \cdot B_{\text{miss}} + |\nt| \cdot \varepsilon_{\text{FP}} \cdot \frac{B_{\text{miss}}}{B_{\text{pin}}}
\end{equation}
Summing over $T$ steps and bounding $\varepsilon_{\text{FN}} + \varepsilon_{\text{FP}} \leq 2\varepsilon$:
\begin{equation}
  \sum_{t=1}^T \expect[\ell_t^{\text{pred}}] \leq O(B_{\text{miss}} \cdot T \cdot \varepsilon)
\end{equation}

Combining both terms yields the stated bound.
\end{proof}

\subsection{Proof of Theorem~\ref{thm:convergence}}

\begin{proof}
Between recompilations (time $t$ to $t + T_r$), the wiki accumulates incremental pin patches.
The quality deviation from a hypothetical full recompilation at each step is bounded by the number of ``stale'' facts that would be reorganized in a full compile.

At each step, at most $|\nt|_{\text{avg}} \cdot \varepsilon$ pins are misplaced (material facts missed or non-material facts pinned).
Over $T_r$ steps, the cumulative misplacement is $T_r \cdot |\nt|_{\text{avg}} \cdot \varepsilon$ tokens.
Normalizing by the pin budget:
\begin{equation}
  \mathcal{Q}(W_{t+T_r}^{\text{incr}}, q) \geq \mathcal{Q}(W_t^{\text{batch}}, q) - c \cdot \frac{T_r \cdot |\nt|_{\text{avg}} \cdot \varepsilon}{B_{\text{pin}}}
\end{equation}
for a constant $c$ depending on the maximum quality sensitivity to a single misplaced pin.

The batch WiCER convergence guarantee~\citep{huerta2026wicer} ensures that each full recompilation converges in at most 2 iterations of the Compile--Evaluate--Refine loop, restoring quality to the batch optimum.
\end{proof}

\subsection{Proof of Proposition~\ref{prop:vanishing}}

\begin{proof}
The time-averaged regret is:
\begin{equation}
  \frac{\text{Regret}(T)}{T} \leq \frac{O(\sqrt{T \log K})}{T} + O(B_{\text{miss}} \cdot \varepsilon) = O\!\left(\sqrt{\frac{\log K}{T}}\right) + O(\varepsilon)
\end{equation}
The first term vanishes as $T \to \infty$.
The second term is bounded by $\varepsilon$, which can be driven to zero with improved volatility prediction (e.g., more training data, online learning of the scorer).
If $\varepsilon = O(T^{-\beta})$ for some $\beta > 0$ (the scorer improves over time), then both terms vanish.
\end{proof}

\section{Exact Prompts for Reproducibility}
\label{app:prompts}

\subsection{Zero-Shot Volatility Scorer Prompt}

The following prompt is used verbatim for the zero-shot scoring path (\S\ref{sec:voltarget}).
The placeholders \texttt{\{ticker\}} and \texttt{\{headline\}} are substituted at inference time.
The prompt is delivered as a user message in the chat template; the system message instructs the model to respond with a single number.

\begin{lstlisting}[caption={Zero-shot volatility scoring prompt (exact text).}]
System:
You are a quantitative financial analyst. Respond with ONLY a
single decimal number between 0 and 1. No explanation.

User:
On a scale of 0 to 1, rate the probability that the following
news headline about {ticker} will cause ABNORMAL stock price
volatility -- defined as realized 5-day return volatility
exceeding 2 times the current cross-sectional average across
all S&P 500 stocks.

Consider: earnings surprises, major M&A activity, CEO changes,
regulatory/legal actions, product recalls, guidance revisions,
activist investor activity, and macroeconomic shocks specific
to this company.

Headline: {headline}

Respond with a single number between 0 and 1.
\end{lstlisting}

\subsection{Fact Extraction Prompt}

The following prompt is used by \textsc{ExtractFacts} to convert a raw headline into a
pinnable fact sentence for the wiki.

\begin{lstlisting}[caption={Fact extraction prompt (exact text).}]
System:
Extract facts concisely. Respond with one sentence only.

User:
Extract the single most important financial fact from this
news about {ticker}:
{headline}
\end{lstlisting}

\subsection{Wiki Compilation Prompt}

Used by \textsc{BatchWiCER} to compile a set of extracted facts into a coherent wiki section.

\begin{lstlisting}[caption={Wiki compilation prompt (exact text).}]
System:
You are a financial encyclopedia editor. Write concise,
factual wiki entries.

User:
Compile the following news about {entity} into a wiki section
covering key events, financial performance, and outlook.
200-300 words, encyclopedia style.

Recent news:
{headlines_list}
\end{lstlisting}

\end{document}